\documentclass[10pt,journal,compsoc]{IEEEtran}
\usepackage{cite}
\usepackage{amsmath,amssymb,amsfonts,amsthm}
\usepackage{algorithm}
\usepackage{algpseudocode}
\makeatletter
\renewcommand{\theALG@line}{\thealgorithm.\arabic{ALG@line}}
\renewcommand{\theHALG@line}{\thealgorithm.\arabic{ALG@line}}
\makeatother
\usepackage{array}
\usepackage{booktabs}
\usepackage{textcomp}
\usepackage{stfloats}
\usepackage{url}
\usepackage{graphicx}
\usepackage{balance}
\usepackage{bm}
\usepackage{xcolor}
\usepackage{multirow}
\usepackage[backref=page]{hyperref}
\usepackage{cleveref}
\newcommand{\R}{\mathbb{R}}

\newcommand{\xclean}{\bm{x}}
\newcommand{\xadv}{\bm{x}^{adv}}
\newcommand{\ytrue}{y}
\newcommand{\deltaadv}{\bm{\delta}}
\newcommand{\vdual}{\bm{v}}
\newcommand{\loss}{J}
\newcommand{\proj}{\mathcal{P}}
\newcommand{\Q}{\mathbf{Q}}
\newcommand{\Qone}{\mathbf{Q}_1}

\newcommand{\Dp}[1]{\,{\scriptsize(#1)}}

\def\BibTeX{{\rm B\kern-.05em{\sc i\kern-.025em b}\kern-.08em
    T\kern-.1667em\lower.7ex\hbox{E}\kern-.125emX}}

\newtheorem{theorem}{Theorem}
\newtheorem{proposition}{Proposition}
\newtheorem{assumption}{Assumption}

\theoremstyle{remark}
\newtheorem{remark}{Remark}

\begin{document}

\title{
Pareto-Improving Adversarial Attacks with Primal-Dual Regularization
}

\author{Yang~Dai,
        Longfei~Zhang,
        Wei~Tao,
        Li~Shen,
        Jincai~Huang,
        and~Qing~Tao
\IEEEcompsocitemizethanks{
\IEEEcompsocthanksitem Y.~Dai, L.~Zhang, W.~Tao and J.~Huang are with Laboratory for Big Data and Decision, National University of Defense Technology, Changsha 410073, China (e-mail: \{daiyang2000, zhanglongfei, taowei, huangjincai\}@nudt.edu.cn).
\IEEEcompsocthanksitem W.~Tao is also with Academy of Military Science, Beijing 100091, China. Corresponding author.
\IEEEcompsocthanksitem L.~Shen is with School of Cyber Science and Technology, Shenzhen Campus of Sun Yat-sen University, Shenzhen 518107, China (e-mail: mathshenli@gmail.com).
\IEEEcompsocthanksitem Q.~Tao is with Hefei Institute of Technology, Hefei 230009, China (e-mail: taoqing@gmail.com).
}}

\markboth{}%
{Dai \MakeLowercase{\textit{et~al.}}: Pareto-Improving Adversarial Attacks with Primal-Dual Regularization}

\maketitle

\begin{abstract}
Transferable adversarial attacks are arguably the most practical black-box threat model. Under the same perturbation budget, stronger transfer attacks attain higher attack success rate (ASR), yet their imperceptibility also tends to degrade. Under such a fixed-budget protocol, transferability and imperceptibility therefore appear to trade off against each other. We argue that this conflict is an artifact of fixed-budget evaluation, not an intrinsic trade-off. When attacks are compared on the ASR--imperceptibility Pareto frontier obtained by sweeping $\epsilon$, stronger transfer attacks already attain better imperceptibility at matched ASR than weaker ones. To exploit this latent advantage, we introduce the stealthy transfer attack ST, a plug-in primal-dual wrapper that adds an $L_\infty$ saturation regularizer to the standard constrained objective and resolves it through a two-step primal-dual update: a projected primal step on the perturbation coupled with an $L_1$-ball projection on a dual variable that absorbs the regularizer through Fenchel duality, requiring no auxiliary models or handcrafted perceptual priors. Empirically, ST extends the Pareto frontier across different base attacks and additional surrogate architectures. At $\epsilon{=}16/255$, average imperceptibility gains over each base attack are $17\%$ on LPIPS and $14\%$ on NIQE while ASR is preserved or improved. At matched high-ASR levels, the strongest ST variants further Pareto-dominate dedicated stealth-oriented transfer attacks, confirming that the latent imperceptibility advantage of strong transfer attacks can be unlocked by a primal-dual optimization wrapper without sacrificing transferability. Code will be made available at \url{https://github.com/AndssY/ST}.

\end{abstract}


\section{Introduction}
\IEEEPARstart{D}{eep} neural networks are vulnerable to adversarial perturbations~\cite{szegedy2014intriguing}~\cite{wu2024aml}. Small changes to a clean input can flip a classifier's prediction with high confidence, even though the change is invisible. Studying such attacks at scale is therefore central to understanding and hardening the decision boundaries of deployed models. Concretely, an adversarial attack is a constrained optimization problem: maximize a prediction loss with respect to a perturbation $\delta$ added to a clean input $x$, subject to an imperceptibility constraint that confines $\delta$ to an $L_\infty$-ball of radius $\epsilon$ and keeps $x+\delta$ in the valid pixel range. 

Among adversarial attacks, transfer attacks are arguably the most practical threat model, as they craft the adversarial example on a surrogate model, relying on the transferability to attack unseen target models. Strong transfer attacks therefore concentrate on the surrogate-side ascent direction, stabilizing it with momentum~\cite{dong2018mifgsm}, diversifying it with input transformations~\cite{xie2019dim,dong2019tim,lin2020nifgsm}, or reshaping it through neighborhood sampling~\cite{qiu2026mef,guo2025ops}. These choices substantially raise the transferability (measured by attack success rate, ASR) while staying inside the same attack budget $\epsilon$.

\begin{figure}[t]
\centering
    \includegraphics[width=\columnwidth]{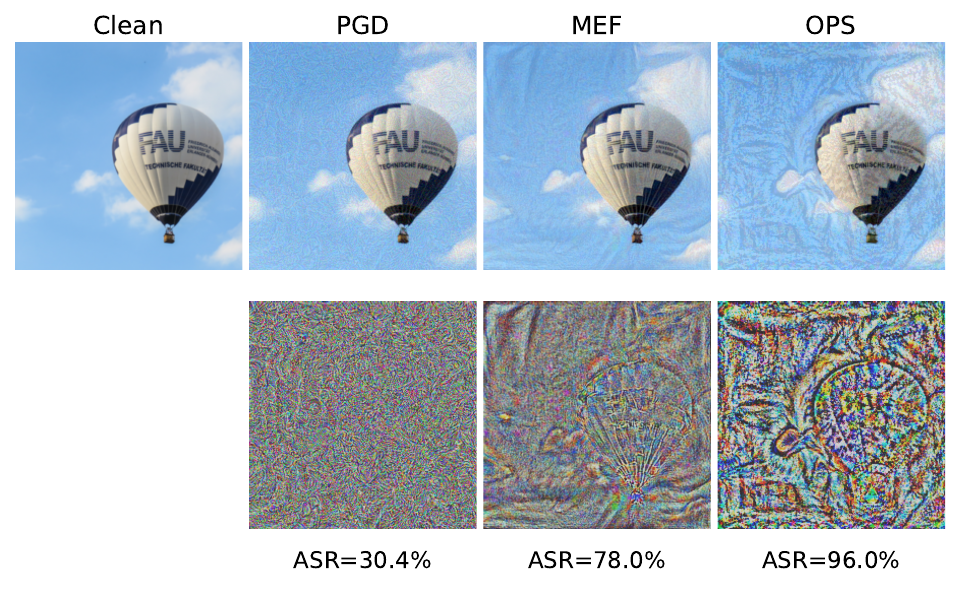}%
\caption{Comparison at the standard $L_\infty{=}16/255$ budget. The adversarial perturbation visualized as $5\times\delta+0.5$. PGD is visually clean but transfers poorly, whereas MEF and OPS reach much higher ASR at the cost of visible saturation artifacts.}
\label{fig:teaser}
\end{figure}

Higher ASR under the same budget does not come for free. The stronger the attack, the worse its imperceptibility. Figure~\ref{fig:teaser} shows this on a representative sample: as ASR rises from PGD to OPS, image quality progressively deteriorates. A recent systematic evaluation~\cite{zhao2025revisiting} confirms this pattern at scale, reporting that almost all representative transfer attacks yield worse imperceptibility than a plain PGD baseline across a panel of complementary perceptual metrics, and calls for imperceptibility to receive much more attention. Read under the standard fixed-$\epsilon$ protocol, transferability and imperceptibility therefore appear to trade off against each other.

This calls for a different lens on the question. Since transferability (measured by ASR) and imperceptibility are two competing objectives, the natural lens is the ASR--imperceptibility Pareto frontier, where each attack is characterized by the full operating curve it traces as the budget $\epsilon$ varies rather than by a single fixed-$\epsilon$ point. One attack Pareto-dominates another if it attains better imperceptibility at the same ASR (or, equivalently, higher ASR at matched imperceptibility). Besides, imperceptibility is hard to quantify through any single stealth metric, so we evaluate Pareto dominance across complementary metric families (full-reference, no-reference, and distributional) rather than any single index.

This Pareto-frontier lens reveals a counterintuitive picture (\cref{subsec:pareto}): stronger attacks already weakly dominate weaker ones. At matched ASR, stronger attacks score better on almost all imperceptibility metrics. The apparent conflict observed at fixed $\epsilon$ is therefore a snapshot artifact, not an intrinsic conflict between transferability and imperceptibility. This points to a different kind of imperceptibility-enhancing attack that preserves the strong base attack's direction-generation pipeline and only modulates how that direction is projected into the next perturbation iterate. Existing stealth-oriented transfer attacks~\cite{hu2024came,du2023ps,guesmi2025tesser,chen2024diffattack,li2024advad,wang2024egstssa} take the opposite path. They augment the base attack with auxiliary signals such as class-activation maps, JND priors, or pretrained diffusion generators, which couple the imperceptibility mechanism to a particular model class and sacrifice modularity. 

We introduce the stealthy transfer attack ST, which is a plug-in primal-dual wrapper that modifies only the update step of a compatible base attack. Concretely, ST adds an $L_\infty$ saturation regularizer to the standard constrained objective and handles it through a two-step primal-dual update: a projected primal step on the perturbation, coupled with an $L_1$-ball projection on a dual variable that absorbs the regularizer through Fenchel duality. The wrapper preserves the base attack's gradient or momentum construction, intervenes only when forming the next perturbation iterate, and adds no auxiliary models, no handcrafted perceptual priors, and no architectural change to the surrogate. 

Empirically, ST extends the ASR--imperceptibility Pareto frontier of seven representative base attacks. Taking STMI versus MI-FGSM as the canonical example, the wrapper maintains competitive ASR while simultaneously improving all reported imperceptibility metrics, reducing LPIPS by $31\%$ and NIQE by $21\%$ and raising MS-SSIM by $3\%$ and TOPIQ-FR by $16\%$. Comparable shifts hold across additional surrogate architectures. Together, these results show that the latent imperceptibility advantage of strong transfer attacks can be unlocked by a primal-dual optimization wrapper without sacrificing transferability.

This paper makes three contributions:
\begin{itemize}
\item \textbf{A Pareto-frontier perspective on transferability vs.\ imperceptibility.} We propose to evaluate transfer attacks on the Pareto frontier rather than at a single fixed perturbation budget, and show that under this lens stronger transfer attacks already weakly dominate weaker ones at matched ASR (\cref{subsec:pareto}).
\item \textbf{ST, a plug-in primal-dual wrapper.} We introduce ST, which modifies only the projection step of a compatible base attack: an added $L_\infty$ saturation regularizer is routed to the dual side through Fenchel
conjugacy and resolved by a single-loop primal-dual update, adding negligible runtime overhead and requiring no auxiliary modules or perceptual priors.
\item \textbf{Broad empirical validation.} We evaluate ST across seven representative base attacks, multiple surrogate architectures, and three stealth-oriented transfer attack baselines. ST extends the Pareto frontier
of every tested base attack (\cref{tab:cross_attack_eps16}, \cref{fig:pareto_stmi}, \cref{fig:pareto_six}); at matched ASR, the strongest ST variants Pareto-dominate the stealth-oriented competitors (\cref{fig:q4_pareto_vit}).
\end{itemize}

\section{Related Work}
\label{sec:related}

Transferable adversarial attacks craft adversarial examples on surrogate models and evaluate them on unseen target models, typically under the standard $L_\infty$ threat model. Most transfer attacks differ in how they estimate, stabilize, or transform the ascent direction under the same perturbation budget. However, the standard $L_\infty$ constraint alone does not ensure imperceptibility, since $L_p$ proximity is insufficient to characterize human perceptual similarity, and adversarial examples under the same norm budget may still differ substantially in visual quality~\cite{sharif2018suitability,zhao2020percal,liu2024stealthiness}. We therefore briefly review representative transfer attacks and then focus on work that explicitly studies imperceptibility and stealthy transfer. Additional notes on stealth evaluation are deferred to Appendix~\ref{subsec:appendix_stealth_eval}.

\textit{a) Transfer attacks.}
Representative transfer attacks and transfer-enhancement strategies include MI-FGSM~\cite{dong2018mifgsm}, MEF~\cite{qiu2026mef}, VMI-FGSM~\cite{wang2021vmi}, EMI-FGSM~\cite{wang2021emi}, GRA~\cite{zhu2023gra}, DIM~\cite{xie2019dim}, TIM~\cite{dong2019tim}, SIM~\cite{lin2020nifgsm}, SGM~\cite{wu2020skip}, RAP~\cite{qin2022rap}, and OPS~\cite{guo2025ops}. These methods all act on the \emph{ascent direction}, e.g., by stabilizing it with momentum, averaging it over input transformations, or reshaping it through neighborhood sampling, and then close the iteration with a standard sign-step plus $L_\infty$ projection. ST is orthogonal in scope: it leaves the ascent direction untouched and changes only the update/projection step that converts that direction into the next iterate, so it can be inserted as a modular wrapper on any of these attacks without altering their direction-generation pipeline.

\textit{b) Stealth-oriented attacks.}
Stealth-oriented adversarial attacks typically improve perceptual quality from two directions. The first uses sparse or structured perturbation support, as in EAD~\cite{chen2018ead}, SparseFool~\cite{modas2019sparsefool}, StrAttack~\cite{xu2019strattack}, $L_1$-APGD~\cite{croce2021l1apgd}, SAIF~\cite{imtiaz2025saif}, SAPF~\cite{fan2020sapf}, and GSE~\cite{zimmer2023gse}. The second replaces or augments pixelwise norm constraints with perceptual priors, as in perceptual color distance~\cite{zhao2020percal} and semantic-similarity wavelet attacks~\cite{luo2022ssah}. These works show that perturbation structure and perceptual priors can materially affect visual noticeability. However, they are primarily evaluated under white-box or minimum-norm settings on a fixed surrogate-target pair, and the question of whether their imperceptibility advantages survive when the adversarial example must transfer to an unseen target model has received much less attention.

\textit{c) Stealthy transfer attacks.}
Stealthy transfer attacks aim to improve both black-box transferability and imperceptibility under the same perturbation budget. Existing methods mainly improve imperceptibility through three routes. The first introduces spatial or perceptual priors into the perturbation design. CAME~\cite{hu2024came} uses ensemble CAM weighting to guide spatial allocation, PS Attack~\cite{du2023ps} incorporates a precomputed JND prior, and TESSER~\cite{guesmi2025tesser} regularizes perturbations with Gaussian spectral smoothing. The second exploits structured sparse generation. EGS-TSSA~\cite{wang2024egstssa} employs a trained generator to learn structured sparse perturbations. The third moves the attack process into a generative latent space. DiffAttack~\cite{chen2024diffattack} perturbs diffusion latent representations, while AdvAD~\cite{li2024advad} models attacking as a non-parametric diffusion process. These methods improve imperceptibility through perceptual or generative priors, but usually at the cost of reduced modularity and higher attack-specific complexity. For instance, EGS-TSSA requires saliency-map computation per iteration, and DiffAttack runs a full diffusion model in the loop—both incur $\sim$10$\times$ generation cost relative to the base attack and tie the imperceptibility mechanism to a particular model class. ST replaces only the projection step, at a negligible runtime overhead per iteration relative to the surrogate's forward and backward passes.

Relatedly, Zhao~\emph{et al.}~\cite{zhao2025revisiting} revisit the evaluation side of this problem by benchmarking 23 representative transfer attacks under the standard $L_\infty{=}16/255$ protocol. They report that many strong transfer attacks appear less imperceptible than a plain PGD baseline under fixed-budget snapshots, and argue that a single $L_\infty$ budget is insufficient for fair perceptual comparison. This critique is primarily evaluative rather than algorithmic, and is closely aligned with the matched-ASR Pareto-frontier perspective adopted in Sec.~\ref{sec:experiments}.

\section{Method}
\label{sec:method}

ST is defined at the level of the perturbation update. The base attack still produces an ascent direction $\bm{u}_t$; ST changes only how that direction is converted into the next iterate. We first write the standard transfer attack objective, and then introduce the regularized formulation behind the wrapper.

Given a classifier $f_{\theta}$ parameterized by $\theta$, crafting a non-targeted adversarial example $\xadv$ from a clean example $\xclean$ can be written as
\begin{equation}
\label{eq:st_attack_x}
\max \loss(f_{\theta}(\xadv), \ytrue),
\quad
\text{s.t.}\quad
\|\xadv-\xclean\|_{\infty} \le \epsilon,
\end{equation}
where $\loss(\cdot,\cdot)$ denotes the classification loss, $\ytrue$ is the ground-truth label, and $\epsilon$ is the perturbation budget. For convenience, we denote the perturbation by $\deltaadv=\xadv-\xclean$ and write the feasible set as
\begin{equation}
\label{eq:st_q}
\Q=\left\{\deltaadv:\|\deltaadv\|_{\infty}\le\epsilon,\; \xclean+\deltaadv\in[0,1]^d\right\}.
\end{equation}
The attack problem becomes
\begin{equation}
\label{eq:st_standard_problem}
\min_{\deltaadv} -J(\deltaadv),
\quad
\text{s.t.}\quad
\deltaadv\in\Q,
\end{equation}
where $J(\deltaadv)=\loss(f_{\theta}(\xclean+\deltaadv),\ytrue)$. The set $\Q$ enforces both the perturbation budget and the image-range constraint, but treats every feasible perturbation as equivalent at the objective level even though feasible perturbations can differ markedly in visual quality. In particular, gradient-driven optimizers tend to saturate $\deltaadv$ near the $\pm\epsilon$ boundary on a substantial fraction of coordinates, which produces the visible color speckle and edge ringing flagged in \cref{fig:teaser}.

ST therefore adds an explicit penalty on this peak-budget saturation. Because the visible artifacts are concentrated on coordinates that hit $\pm\epsilon$, we penalize the largest absolute coordinate of $\deltaadv$ directly, i.e., the $L_\infty$ norm itself, rather than a global magnitude such as $\|\deltaadv\|_2$ or a perceptual prior that would require an auxiliary network. We therefore consider
\begin{equation}
\label{eq:st_regularized_problem}
\min_{\deltaadv}\; -J(\deltaadv) + \lambda\|\deltaadv\|_{\infty},
\quad
\text{s.t.}\quad
\deltaadv\in\Q,
\end{equation}
where $\lambda\ge 0$ controls the strength of the additional regularization; we sweep $\bar\lambda \in [0.1, 5.0]$ in \cref{subsec:ablation}.

To expose the structure more clearly, we rewrite \eqref{eq:st_regularized_problem} as
\begin{equation}
\label{eq:st_main_problem}
\min_{\deltaadv}\; -J(\deltaadv) + \lambda\|\deltaadv\|_{\infty} + \iota_{\Q}(\deltaadv),
\end{equation}
where $\iota_{\Q}$ is the indicator function of $\Q$. The objective now contains two non-smooth terms, $\lambda\|\deltaadv\|_{\infty}$ and $\iota_{\Q}(\deltaadv)$. ST handles them through a primal-dual reformulation rather than combining them inside a single primal proximal step.

Let $\Qone$ denote the $L_1$ ball
\begin{equation}
\label{eq:st_q1}
\Qone=\left\{\bm{z}:\|\bm{z}\|_1\le\lambda\right\}.
\end{equation}
By Fenchel duality, the support function of the $L_1$ ball equals the $L_\infty$ norm,
\begin{equation}
\lambda\|\deltaadv\|_{\infty}=\max_{\bm{z}\in\Qone}\langle \deltaadv,\bm{z}\rangle.
\end{equation}
Substituting this identity into \eqref{eq:st_main_problem} converts the non-smooth $L_\infty$ regularizer into the saddle problem
\begin{equation}
\label{eq:st_saddle}
\min_{\deltaadv}\max_{\bm{z}\in\Qone}\,\bigl[-J(\deltaadv) + \iota_{\Q}(\deltaadv) + \langle \deltaadv, \bm{z}\rangle\bigr],
\end{equation}
which separates the two non-smooth terms onto different variables: the feasibility indicator $\iota_{\Q}$ stays on the primal side, while the regularizer is now carried by the constraint $\bm{z}\in\Qone$ on the dual side. This separation is the design choice that lets us avoid combining $\iota_{\Q}$ and $\lambda\|\cdot\|_{\infty}$ inside a single primal proximal step, which would have no closed form.

Problem~\eqref{eq:st_saddle} fits the primal-dual splitting of Condat~\cite{condat2013primaldual} and V\~{u}~\cite{vu2013splitting} with identity coupling $K{=}I$, and the resulting proximal maps reduce to two closed-form Euclidean projections, $\proj_{\Q}$ and $\proj_{\Qone}$ (Proposition~\ref{prop:prox_maps}). The full derivation is given in Appendix~\ref{sec:appendix_theory}. Within the broader family of ADMM and primal-dual methods studied for machine learning~\cite{lin2011ladmap,lin2022admm_book}, the Condat--V\~{u} splitting is convenient here because our $L_\infty$-saturation regularizer admits a closed-form $L_1$-ball projection without further linearization.

In the practical implementation, the dual-ball radius is specified through a dimension-free scalar $\bar\lambda$ and the absolute radius entering $\Qone$ is recovered as $\lambda=\bar\lambda\,d$, where $d=C\times H\times W$ is the perturbation dimensionality (e.g., $d=3\times 224\times 224$ on ImageNet). This convention keeps the per-coordinate budget $\lambda/d=\bar\lambda$ constant across input resolutions, so a single $\bar\lambda$ value transfers between datasets without rescaling. Throughout the paper, the regularization-strength values reported in figures and tables refer to $\bar\lambda$; the appendix details the convention.

Let $\bm{u}_t \in \R^d$ denote the ascent direction produced by the base attack at iteration $t$, and let $\vdual_t$ denote a dual variable that tracks the $L_1$-ball constraint. The ST update rule is
\begin{align}
\bar{\delta}_{t+1} &= \proj_{\Q}\!\big(\deltaadv_t + \alpha \bm{u}_t - \alpha \vdual_t \big), \label{eq:st_primal} \\
\bar{\vdual}_{t+1} &= \proj_{\Qone}\!\big(\vdual_t + \beta (2\bar{\delta}_{t+1} - \deltaadv_t)\big), \label{eq:st_dual} \\
\deltaadv_{t+1} &= (1-\rho)\deltaadv_t + \rho \bar{\delta}_{t+1}, \label{eq:st_relax_primal} \\
\vdual_{t+1} &= (1-\rho)\vdual_t + \rho \bar{\vdual}_{t+1}. \label{eq:st_relax_dual}
\end{align}
The primal step in \eqref{eq:st_primal} applies the base direction together with a dual correction $-\alpha\vdual_t$ and then projects back to $\Q$. The dual step projects an extrapolated state onto $\Qone$ through the sort-based algorithm~\cite{duchi2008l1projection}. The relaxation parameter $\rho$ interpolates between the previous and updated iterates.

The wrapper leaves the direction-generation pipeline intact. When $\bm{u}_t$ is the MI-FGSM momentum, the resulting method is STMI. When $\bm{u}_t$ is obtained from input transformations, ensembles, or other gradient-processing rules, the same wrapper yields the corresponding ST variant without altering the rest of the attack. The MDCS-based variants STMDCSMI and STMDCSMEF are the one structural exception: their underlying base attack carries an additional per-coordinate adaptive-step state $\bm{d}_t$ that is element-wise multiplied into $\bm{u}_t$ before the primal projection, so they require a small extension of Algorithm~\ref{alg:st}. The full update is given as Algorithm~\ref{alg:st_mdcs} in Appendix~\ref{subsec:appendix_stmdcs}, and reduces to Algorithm~\ref{alg:st} in the limit $\gamma\to+\infty$, in which $\bm{d}_t$ remains the all-ones tensor for all $t$.

\begin{algorithm}[t]
\caption{ST update}
\label{alg:st}
\begin{algorithmic}[1]
\Require Clean image $\xclean$, label $\ytrue$, feasible set $\Q$, regularization strength $\lambda$ ($\Qone{=}\{\bm{z}:\|\bm{z}\|_1{\le}\lambda\}$), step sizes $\alpha,\beta$, relaxation $\rho$, iteration number $T$
\State Initialize $\deltaadv_0=0$, $\vdual_0=0$
\For{$t=0,\dots,T-1$}
\State $\bm{g}_t=\nabla_{\deltaadv}J(\deltaadv_t)$
\State Compute $\bm{u}_t$ from $\bm{g}_t$ 

\Comment{e.g., momentum or transformed-gradient average}
\State $\bar{\delta}_{t+1}=\proj_{\Q}\!\left(\deltaadv_t+\alpha\bm{u}_t-\alpha\vdual_t\right)$
\State $\bar{\vdual}_{t+1}=\proj_{\Qone}\!\left(\vdual_t+\beta(2\bar{\delta}_{t+1}-\deltaadv_t)\right)$
\State $\deltaadv_{t+1}=(1-\rho)\deltaadv_t+\rho\bar{\delta}_{t+1}$
\State $\vdual_{t+1}=(1-\rho)\vdual_t+\rho\bar{\vdual}_{t+1}$
\EndFor
\State \Return $\xclean+\deltaadv_T$
\end{algorithmic}
\end{algorithm}
\section{Experiments}
\label{sec:experiments}

The empirical study is organized around four questions, each of which targets a hypothesis that should hold if the wrapper is doing more than dialing down the perturbation budget. \cref{subsec:pareto} asks whether ST shifts the ASR--imperceptibility Pareto frontier outward of the base attack, rather than reproducing operating points already on its $\epsilon$ sweep. \cref{subsec:competitor} asks whether the same shift survives when the surrogate is changed from a CNN to a ViT, and whether the ST-wrapped curves Pareto-dominate dedicated stealth-oriented transfer attacks at matched ASR. \cref{subsec:ablation} isolates the contribution of the three primal-dual hyperparameters $\lambda$, $\beta$, $\rho$ and contrasts the $L_\infty$ regularizer against $L_1$ and $L_2$ alternatives within the same Condat--V\~{u} template. \cref{subsec:visualization} reports qualitative controls.

\subsection{Experimental Setup}
\label{subsec:setup}
Following the standard TransferAttack evaluation setup~\cite{zhao2025revisiting}, we use $1000$ randomly sampled images from the NIPS2017 ImageNet-compatible benchmark. Unless otherwise specified, we use a perturbation budget of $\epsilon{=}16/255$, $T{=}10$ attack iterations, and a primal step size of $\alpha{=}1.6/255$.\footnote{For the PGD baseline~\cite{madry2018pgd}, we use $\alpha=2.0/255$. In the $\epsilon$ sweep, we follow the standard PGD setting and set its step size to $1.25\times$ the default step size used by the other methods.}

\noindent\textbf{ASR and Target models}
For black-box transfer evaluation, we consider seven target models spanning both CNN and ViT architectures: VGG-16, MobileNet-V2, Inception-V3, ViT-Base, PiT-B, Visformer-S, and Swin-T. Following the default TransferAttack convention, the main-text ASR is reported as the average over these seven targets together with the surrogate model, and is therefore a mixed-target average rather than a purely transfer-only metric. Unless otherwise stated, ResNet-50 serves as the surrogate model in this section.

\noindent\textbf{Stealth metrics}
As there is no universally perfect metric for evaluating imperceptibility, we report three classes of quality metrics to provide a more comprehensive evaluation from complementary perspectives, including perceptual fidelity, naturalness, and distributional similarity. Full-reference metrics compare the adversarial image with its clean counterpart: PSNR, MS-SSIM~\cite{wang2003msssim}, LPIPS~\cite{zhang2018lpips}, DISTS~\cite{ding2020dists}, TOPIQ-FR~\cite{chen2024topiq}, and LF~\cite{luo2022ssah}. No-reference metrics assess the adversarial image alone, including NIQE~\cite{mittal2013niqe} and CLIP-IQA~\cite{wang2023clipiqa}. We further report FID~\cite{heusel2017fid} as a distributional metric. Appendix~\ref{subsec:appendix_stealth_eval} gives a brief discussion of why this evaluation is inherently multi-metric.

\subsection{Pareto Frontier Analysis}
\label{subsec:pareto}
Reducing the perturbation budget $\epsilon$ is the most obvious way to improve imperceptibility, so it is the natural baseline for asking \textbf{whether ST does anything beyond producing a weaker attack}. If smaller $\epsilon$ were enough to explain the effect, the ST curve would simply reproduce operating points already available on the base attack's budget sweep.

We test this by sweeping $\epsilon$ from $2/255$ to $32/255$ for each ST--base pair, so that both the wrapper and the budget-reduction baseline trace out a full ASR--imperceptibility curve rather than a single operating point. Figure~\ref{fig:pareto_stmi} compares STMI with MI-FGSM~\cite{dong2018mifgsm}. Across the eight-metric view, the STMI curve lies on the favorable side of the MI-FGSM curve over the entire reported sweep: at matched ASR it improves the shown stealth metrics, and at matched imperceptibility levels it retains higher ASR. The wrapper therefore moves the operating curve outward, rather than tracing operating points already on the base attack's $\epsilon$-sweep curve.
\begin{figure}[t]
\centering
    \includegraphics[width=\columnwidth]{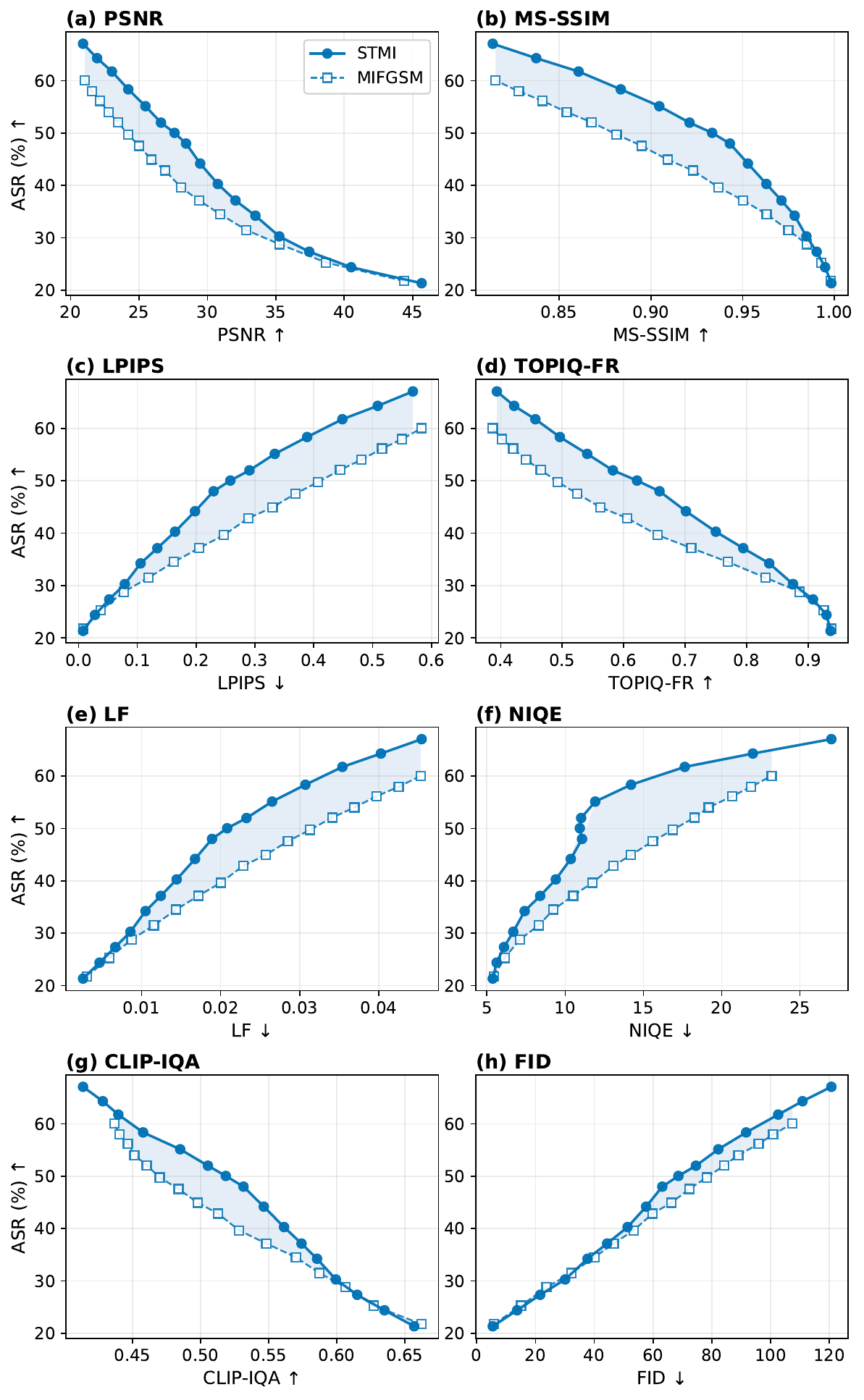}%
\caption{ASR--imperceptibility curves for STMI and MI-FGSM. Each point corresponds to one $\epsilon$ value. Solid line with filled circles denotes STMI; dashed line with hollow squares denotes MI-FGSM. The light-blue shading marks the envelope between the two curves (the operating-gap region). Across the reported sweep, the STMI curve lies on the favorable side of MI-FGSM for the shown metrics, which suggests that ST changes the operating curve rather than merely moving along the base attack's budget sweep.}
\label{fig:pareto_stmi}
\end{figure}

We repeat the same experiment for four additional base attacks: the non-momentum baseline PGD~\cite{madry2018pgd}, the input-transformation attack DIM~\cite{xie2019dim}, the high-ASR attack OPS~\cite{guo2025ops}, and the MDCS-based attack MDCSMEF~\cite{qiu2026mef,tao2026mdcs}. As shown in Figure~\ref{fig:pareto_six}, the ST variant occupies better operating points than the corresponding base attack across most of the sweep for each method. This indicates that ST is a modular wrapper, as it can be attached to qualitatively different attack types and still shift their ASR--imperceptibility frontier outward.
\begin{figure*}[t]
\centering
    \includegraphics[width=\textwidth]{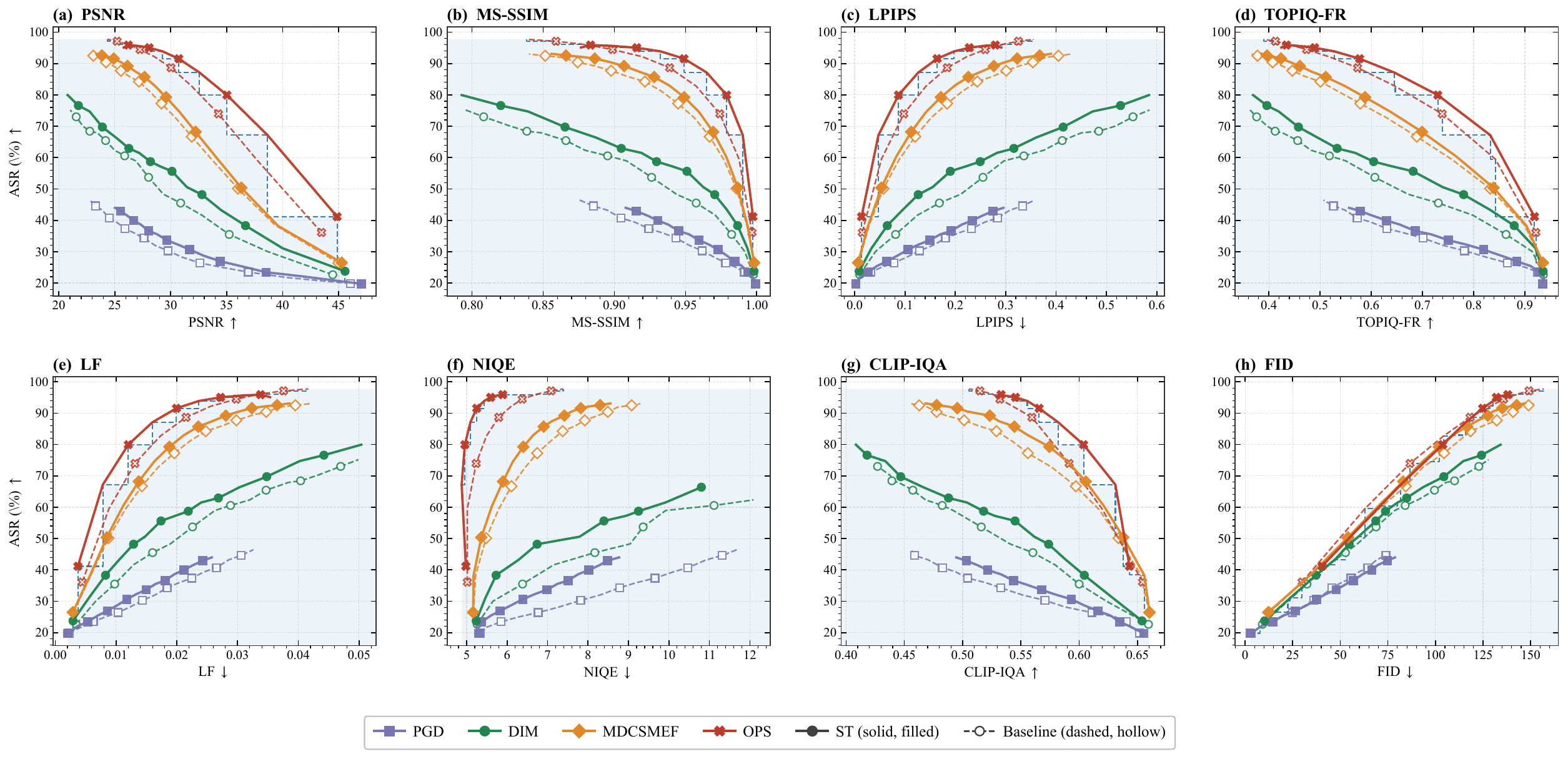}%
\caption{ASR--imperceptibility curves across eight stealth metrics. Solid lines with filled markers denote ST variants, and dashed lines with hollow markers denote the corresponding base attacks. Light-blue shading indicates the Pareto-dominated region. Within each method, the ST-wrapped curve shifts to a strictly better operating region  over most of the sweep. Across methods, stronger transfer attacks such as STOPS tend to Pareto-dominate weaker ones at matched ASR.}
\label{fig:pareto_six}
\end{figure*}

Figure~\ref{fig:pareto_six} also clarifies a point that is obscured by a fixed-budget comparison. \textbf{When attacks differ in transferability, they are more appropriately compared on the Pareto frontier than at the same $\epsilon$.} At a shared budget, one attack may achieve higher ASR and worse imperceptibility simply because it is stronger, while another may look cleaner only because its ASR is much lower. PGD indeed appears more stealthy than stronger transfer attacks at $\epsilon{=}16/255$, but that operating point corresponds to only about $32\%$ ASR. The meaningful comparison is therefore at matched ASR, where we ask which attack is more stealthy, or equivalently at matched imperceptibility, where we ask which attack reaches higher ASR. Under this Pareto-frontier view, the fixed-$\epsilon$ argument that PGD is more stealthy than stronger transfer attacks~\cite{zhao2025revisiting} no longer holds: Figure~\ref{fig:pareto_six} shows that stronger transfer attacks reach PGD's $32\%$-ASR operating point at substantially smaller budgets and already dominate PGD on the reported stealth metrics there. PGD's apparent imperceptibility advantage is therefore an artifact of the fixed-budget protocol; at matched ASR, stronger transfer attacks are in fact more stealthy than PGD.

\textbf{Viewed from this perspective, transferability and imperceptibility need not be framed as inherently competing goals: a stronger transferable attack can already reach a target ASR with a smaller perturbation budget, and therefore operate in a better perceptual regime.} ST exploits this property by leaving the base attack's direction-generation machinery (momentum accumulation, input transformation, MDCS-type step scheduling) intact and intervening only at the update/projection stage, so its role is to improve the imperceptibility profile achievable by that attack at a given ASR, or equivalently to retain more ASR at a given imperceptibility level.

\begin{table*}[t]
	\caption{ST versus base attack at $\epsilon{=}16/255$ on seven representative attacks. Best per row in bold; the percentage in subscript reports the relative change of the ST variant over its base attack.}
	\label{tab:cross_attack_eps16}
	\centering
	\resizebox{\textwidth}{!}{%
	\begin{tabular}{@{}l l l l l l l l l l l@{}}
		\toprule
		\multicolumn{2}{c}{} &
		\multicolumn{6}{c}{Full-reference} &
		\multicolumn{2}{c}{No-reference} &
		\multicolumn{1}{c}{Distributional} \\
		\cmidrule(lr){3-8}\cmidrule(lr){9-10}\cmidrule(lr){11-11}
		Method &
		\small ASR $\uparrow$ (\%) &
		\small PSNR $\uparrow$ &
		\small MS-SSIM $\uparrow$ &
		\small LPIPS $\downarrow$ &
		\small DISTS $\downarrow$ &
		\small TOPIQ-FR $\uparrow$ &
		\small LF $\downarrow$ &
		\small NIQE $\downarrow$ &
		\small CLIP-IQA $\uparrow$ &
		\small FID $\downarrow$ \\
		\midrule
		PGD &
		31.98 & 28.61 & 0.953 & 0.155 & 0.172 & 0.741 & 0.0164 & 8.29 & 0.552 & \textbf{40.89} \\
		STPGD &
		\textbf{32.31}\Dp{+1.0\%} &
		\textbf{30.61}\Dp{+7.0\%} &
		\textbf{0.966}\Dp{+1.4\%} &
		\textbf{0.128}\Dp{-17.4\%} &
		\textbf{0.158}\Dp{-8.1\%} &
		\textbf{0.784}\Dp{+5.8\%} &
		\textbf{0.0134}\Dp{-18.3\%} &
		\textbf{6.69}\Dp{-19.3\%} &
		\textbf{0.578}\Dp{+4.7\%} &
		42.52\Dp{+4.0\%} \\
		\addlinespace[4pt]
		MI-FGSM &
		42.86 & 26.92 & 0.923 & 0.289 & 0.215 & 0.606 & 0.0228 & 13.08 & 0.513 & 59.85 \\
		STMI &
		\textbf{44.19}\Dp{+3.1\%} &
		\textbf{29.48}\Dp{+9.5\%} &
		\textbf{0.953}\Dp{+3.2\%} &
		\textbf{0.198}\Dp{-31.4\%} &
		\textbf{0.189}\Dp{-12.1\%} &
		\textbf{0.701}\Dp{+15.7\%} &
		\textbf{0.0168}\Dp{-26.6\%} &
		\textbf{10.36}\Dp{-20.8\%} &
		\textbf{0.546}\Dp{+6.6\%} &
		\textbf{57.65}\Dp{-3.7\%} \\
		\addlinespace[4pt]
		MDCSMI &
		46.66 & 25.57 & 0.903 & 0.343 & 0.240 & 0.547 & 0.0266 & 23.74 & 0.488 & 73.25 \\
		STMDCSMI &
		\textbf{46.85}\Dp{+0.4\%} &
		\textbf{26.29}\Dp{+2.8\%} &
		\textbf{0.913}\Dp{+1.2\%} &
		\textbf{0.317}\Dp{-7.5\%} &
		\textbf{0.229}\Dp{-4.5\%} &
		\textbf{0.574}\Dp{+4.9\%} &
		\textbf{0.0244}\Dp{-8.3\%} &
		\textbf{19.88}\Dp{-16.3\%} &
		\textbf{0.500}\Dp{+2.4\%} &
		\textbf{72.23}\Dp{-1.4\%} \\
		\addlinespace[4pt]
		DIM &
		57.05 & 26.89 & 0.912 & 0.296 & 0.224 & 0.560 & 0.0263 & 10.34 & 0.499 & 76.38 \\
		STDIM &
		\textbf{57.20}\Dp{+0.3\%} &
		\textbf{29.08}\Dp{+8.2\%} &
		\textbf{0.941}\Dp{+3.2\%} &
		\textbf{0.220}\Dp{-25.5\%} &
		\textbf{0.201}\Dp{-10.4\%} &
		\textbf{0.642}\Dp{+14.7\%} &
		\textbf{0.0197}\Dp{-25.2\%} &
		\textbf{8.97}\Dp{-13.2\%} &
		\textbf{0.526}\Dp{+5.4\%} &
		\textbf{71.25}\Dp{-6.7\%} \\
		\addlinespace[4pt]
		GRA &
		74.69 & 26.62 & 0.902 & 0.262 & 0.240 & 0.522 & 0.0283 & 7.88 & 0.524 & 97.45 \\
		STGRA &
		\textbf{75.15}\Dp{+0.6\%} &
		\textbf{27.83}\Dp{+4.6\%} &
		\textbf{0.929}\Dp{+3.0\%} &
		\textbf{0.213}\Dp{-18.6\%} &
		\textbf{0.226}\Dp{-5.6\%} &
		\textbf{0.559}\Dp{+7.1\%} &
		\textbf{0.0242}\Dp{-14.5\%} &
		\textbf{7.43}\Dp{-5.6\%} &
		\textbf{0.543}\Dp{+3.6\%} &
		\textbf{95.76}\Dp{-1.7\%} \\
		\addlinespace[4pt]
		MDCSMEF &
		81.16 & 28.10 & 0.933 & 0.214 & 0.236 & 0.535 & 0.0222 & 7.08 & 0.542 & 112.28 \\
		STMDCSMEF &
		\textbf{83.00}\Dp{+2.3\%} &
		\textbf{28.56}\Dp{+1.6\%} &
		\textbf{0.939}\Dp{+0.6\%} &
		\textbf{0.199}\Dp{-7.0\%} &
		\textbf{0.235}\Dp{-0.6\%} &
		\textbf{0.546}\Dp{+2.2\%} &
		\textbf{0.0212}\Dp{-4.6\%} &
		\textbf{6.65}\Dp{-6.2\%} &
		\textbf{0.556}\Dp{+2.5\%} &
		\textbf{109.12}\Dp{-2.8\%} \\
		\addlinespace[4pt]
		OPS &
		\textbf{96.10} & 26.18 & 0.879 & 0.293 & 0.276 & 0.440 & 0.0337 & 6.75 & 0.526 & 143.21 \\
		STOPS &
		95.66\Dp{-0.5\%} &
		\textbf{27.00}\Dp{+3.1\%} &
		\textbf{0.897}\Dp{+2.1\%} &
		\textbf{0.257}\Dp{-12.2\%} &
		\textbf{0.266}\Dp{-3.5\%} &
		\textbf{0.456}\Dp{+3.7\%} &
		\textbf{0.0309}\Dp{-8.2\%} &
		\textbf{5.76}\Dp{-14.5\%} &
		\textbf{0.535}\Dp{+1.8\%} &
		\textbf{136.01}\Dp{-5.0\%} \\
		\midrule
		\textit{Mean $\Delta$} &
		$+1.0\%$ &
		$+5.3\%$ &
		$+2.1\%$ &
		$-17.1\%$ &
		$-6.4\%$ &
		$+7.7\%$ &
		$-15.1\%$ &
		$-13.7\%$ &
		$+3.9\%$ &
		$-2.5\%$ \\
		\bottomrule
\end{tabular}}

\end{table*}

To pin down the per-method gain at a fixed snapshot of the frontier, Table~\ref{tab:cross_attack_eps16} reports all nine stealth metrics together with ASR at $\epsilon{=}16/255$ on seven base attacks, extending the five attacks from Figures~\ref{fig:pareto_stmi} and~\ref{fig:pareto_six} with MDCSMI~\cite{tao2026mdcs} and GRA~\cite{zhu2023gra}. Three patterns emerge. First, ST consistently improves all stealth metrics across every base attack, with the sole exception of FID on STPGD. The largest gains concentrate on LPIPS ($17.1\%$), LF ($15.1\%$), and NIQE ($13.7\%$), whereas MS-SSIM changes the least ($2.1\%$). This confirms that the ST wrapper applies broadly across different types of transfer attacks. Second, FID is the least consistent metric: six of the seven pairs improve it, whereas STPGD is the only pair that incurs a small FID regression ($+1.6$) while achieving sizable gains on the other reported stealth metrics. This is expected, as FID measures distributional distance in Inception feature space and is less sensitive to the local pixel-level artifacts that ST directly targets than per-image metrics. Third, six of the seven variants also improve ASR. The one exception, STOPS, gives up only $0.44$\,pp of ASR from an already saturated $96\%$ baseline while delivering sizable imperceptibility gains, so we read this point as an acceptable trade-off near the saturation regime rather than a failure of the wrapper. Together with \cref{fig:pareto_six}, these results confirm that ST improves imperceptibility without sacrificing transferability.

\subsection{Comparison with Stealthy Transfer Methods}
\label{subsec:competitor}

We now ask two questions in one experiment: \emph{(i)}~does the ST-vs.-base-attack pattern established on a ResNet-50 surrogate still hold under a different surrogate architecture, and \emph{(ii)}~how do the ST-wrapped attacks compare with recent stealth-oriented transfer attacks at matched ASR? Answering both at once requires a surrogate compatible with all evaluated attacks: TESSER~\cite{guesmi2025tesser} is only defined on ViT-based surrogates, so we run the comparison on a ViT-B/16 surrogate and sweep $\epsilon\in[2,32]/255$ for nine attack variants while otherwise following the same hyperparameters as before. The competitor set is PS Attack~\cite{du2023ps}, CAME~\cite{hu2024came}, and TESSER~\cite{guesmi2025tesser}; the ST set is STPGD, STMDCSMEF, and STOPS together with their paired baselines. Results are reported in Figure~\ref{fig:q4_pareto_vit}.

\begin{figure*}[t]
  \centering
  \includegraphics[width=0.98\textwidth]{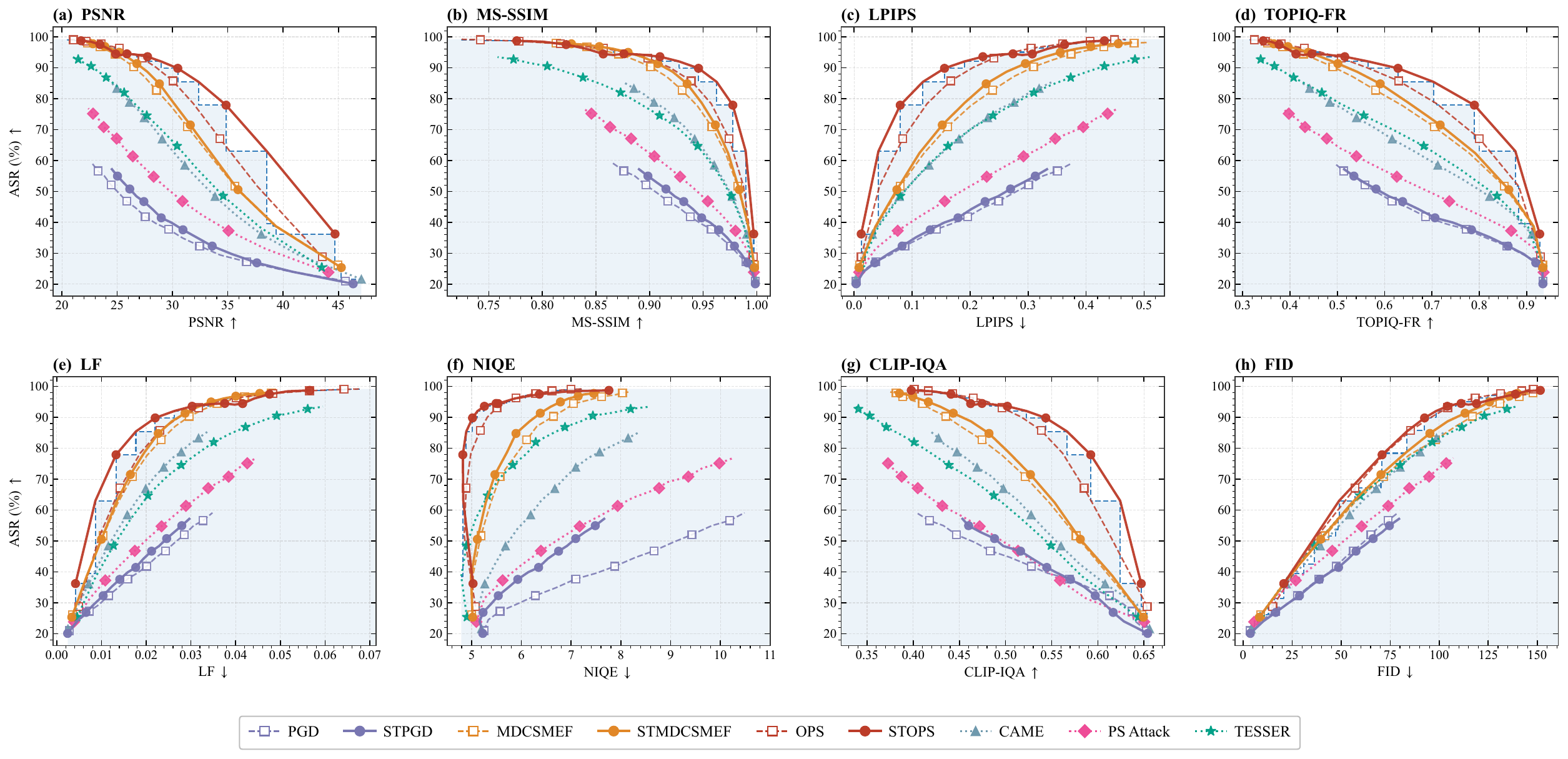}
  \caption{ASR--imperceptibility curves across eight stealth metrics on a ViT-B/16 surrogate. Dotted lines denote dedicated stealth-oriented competitors (CAME, PS Attack, TESSER). STOPS occupies the outer Pareto frontier on most shown metrics. STMDCSMEF plays a similar role one rung lower on the ASR axis.}
  \label{fig:q4_pareto_vit}
\end{figure*}

\noindent\textbf{Generality across additional surrogate architectures} The strongest ST-wrapped attacks again occupy the best operating region: at comparable ASR, STOPS and STMDCSMEF generally improve the shown stealth metrics over the dedicated stealth-oriented competitors, while at comparable imperceptibility levels they typically retain higher ASR. This metric-wise pattern is also consistent with Figure~\ref{fig:pareto_stmi} rather than in tension with it. At the shared snapshot $\epsilon{=}16/255$, averaging over the three shown ST--base pairs, the clearest gains again appear on NIQE ($8.8\%$), LPIPS ($7.1\%$), and LF ($6.4\%$), whereas MS-SSIM changes the least ($1.0\%$) and FID remains the least consistent metric because STPGD incurs a small regression. Relative to the corresponding ResNet-50 results, however, these imperceptibility gains are visibly smaller in magnitude. This suggests that the operating points transferred from ResNet-50 are not yet fully tuned for the ViT-B/16 surrogate. Overall, after switching to ViT-B/16, ST still improves the Pareto frontier of each base attack, but the outward shift is weaker than on ResNet-50, which points to surrogate-specific tuning headroom.

\noindent\textbf{Comparison with stealth-oriented transfer competitors} Among the dedicated competitors, CAME and TESSER are stronger than PS Attack and reach a higher ASR band than STPGD. At a matched ASR level around $63\%$, CAME (LPIPS $0.156$, NIQE $6.40$) and TESSER (LPIPS $0.162$, NIQE $5.32$) are clearly more stealthy than PS Attack (LPIPS $0.320$, NIQE $8.23$); both, however, are still strictly behind STMDCSMEF (LPIPS $0.113$, NIQE $5.28$) and STOPS (LPIPS $0.042$, NIQE $4.83$) at that same ASR level. The same ordering holds across the rest of the eight stealth metrics shown in Figure~\ref{fig:q4_pareto_vit}. The take-away is that the main advantage of ST is not to push imperceptibility at a weak operating point, but to preserve sizable imperceptibility gains while inheriting the transferability of a strong base attack: the ST variants reach a high-ASR regime that the dedicated stealth-oriented competitors do not enter, and at the highest ASR levels where they overlap, the ST variants Pareto-dominate them.

\subsection{Ablation Study}
\label{subsec:ablation}
We next examine which parts of ST matter in practice. The ablation has two components: sensitivity to the three primal-dual hyperparameters $\lambda$, $\beta$, and $\rho$, and a comparison among several choices of the \emph{soft regularizer} $\|\deltaadv\|_{\infty}$ within the same Condat--V\~{u} template. Unless otherwise stated, all experiments in this subsection use STMI with $\bar\lambda{=}2.4$, $\beta{=}160$, $\rho{=}1.0$, $\epsilon{=}16/255$, and a ResNet-50 surrogate.

\textbf{$\bar\lambda$ is the main trade-off parameter}
$\bar\lambda$ controls the radius of the dual $L_1$ ball and has the clearest effect on the final operating point. Sweeping $\bar\lambda \in [0.1, 5.0]$ on STMI traces a wide arc on every stealth metric. As $\bar\lambda$ increases from small to moderate values, the stealth metrics improve substantially while ASR remains relatively stable. The most favorable trade-off appears in the range $\bar\lambda\in[0.5,3.0]$, where imperceptibility is already improved and ASR even peaks slightly. For larger values, the regularizer becomes too strong and both ASR and perturbation magnitude collapse.

To check whether $\bar\lambda$ merely acts as a disguised budget reduction, we also sweep $\epsilon \in [2/255, 16/255]$ on plain MI-FGSM. Across the shown metrics, the STMI $\bar\lambda$ sweep remains more favorable than the MI-FGSM $\epsilon$ sweep at comparable imperceptibility levels. We therefore interpret $\bar\lambda$ as a genuine geometric control on the update, not simply another way of dialing down the perturbation budget. For the cross-attack results in Table~\ref{tab:cross_attack_eps16}, $\bar\lambda$ is selected per base attack from the same sweep range; the per-attack values are listed in Appendix~\ref{sec:appendix_protocol}.

\begin{figure}[h]
\centering
    \includegraphics[width=\columnwidth]{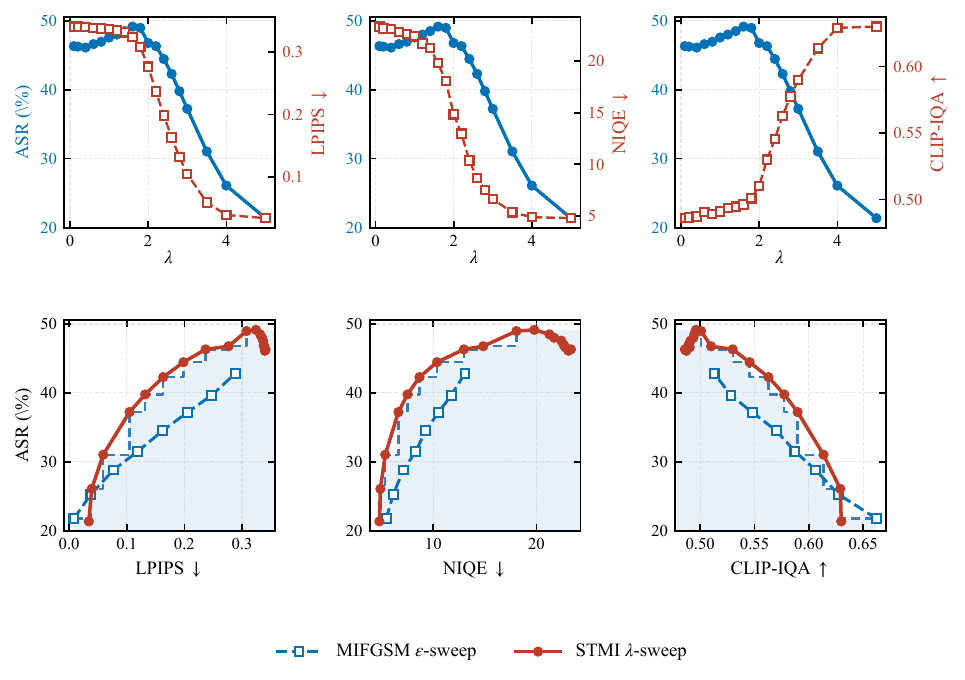}%
\caption{$\bar\lambda$ sensitivity of STMI. Top row: $\bar\lambda$ sweep at fixed $\epsilon{=}16/255$; each panel shows average ASR (left axis, blue solid) and one stealth metric (right axis, red dashed) as $\bar\lambda$ varies. Bottom row: the resulting STMI $\bar\lambda$ sweep in the ASR--imperceptibility plane, compared with the same MI-FGSM $\epsilon$ sweep used as reference. The $\bar\lambda$ sweep traces the widest trade-off arc and reaches more favorable trade-offs than merely shrinking $\epsilon$.}
\label{fig:lambda_vs_eps}
\end{figure}

\textbf{$\beta$ offers a broad usable range}
$\beta$ controls how strongly the dual feedback influences the primal update. We sweep $\beta \in \{20,40,\ldots,300\}$. As presented in Figure~\ref{fig:hparam_beta}, unlike $\bar\lambda$, changing $\beta$ mostly refines a narrower local trade-off band around the default operating region rather than tracing a wide new frontier. In practice, a broad interval $\beta\in[60,260]$ remains usable. The default $\beta=160$ comes from the coupling heuristic $\alpha\beta\lesssim 1$, which here gives $\beta\approx 1/\alpha\approx 160$, and the sweep places that choice close to the trade-off knee. More concretely, Appendix~\ref{sec:appendix_theory} analyzes the linearized primal-dual coupling and shows that the product $c=\alpha\beta$ is the key stability quantity, with $c\approx1$ marking the practical coupling boundary in that reduced model (Proposition~\ref{prop:linear_coupling}; Remark~\ref{rem:beta_160}).

\textbf{$\rho=1.0$ is a conservative default}
$\rho \in (0,2)$ is the over-relaxation parameter of the Condat--V\~{u} iteration. We sweep $\rho \in \{0.2,0.4,\ldots,2.0\}$. As presented in Figure~\ref{fig:hparam_rho}, changing $\rho$ mostly refines a local trade-off segment rather than creating a wider frontier. The range $\rho\in[0.8,1.4]$ is stable in practice, with a mild ASR peak near $\rho=1.2$ at some cost in imperceptibility. As $\rho$ approaches $2.0$, the iteration becomes too aggressive and both ASR and imperceptibility deteriorate. We therefore keep $\rho=1.0$ as a more conservative default.

Taken together, these sweeps indicate that $\bar\lambda$ is the main trade-off knob across ST-wrapped attacks, whereas $\beta$ and $\rho$ behave more like template-level hyperparameters with relatively broad usable ranges. Therefore, we use $(\beta,\rho)=(160,1.0)$ as the default template-level setting across ST variants, while selecting $\bar\lambda$ separately for each wrapped attack.

\begin{figure}[h]
\centering
    \includegraphics[width=\columnwidth]{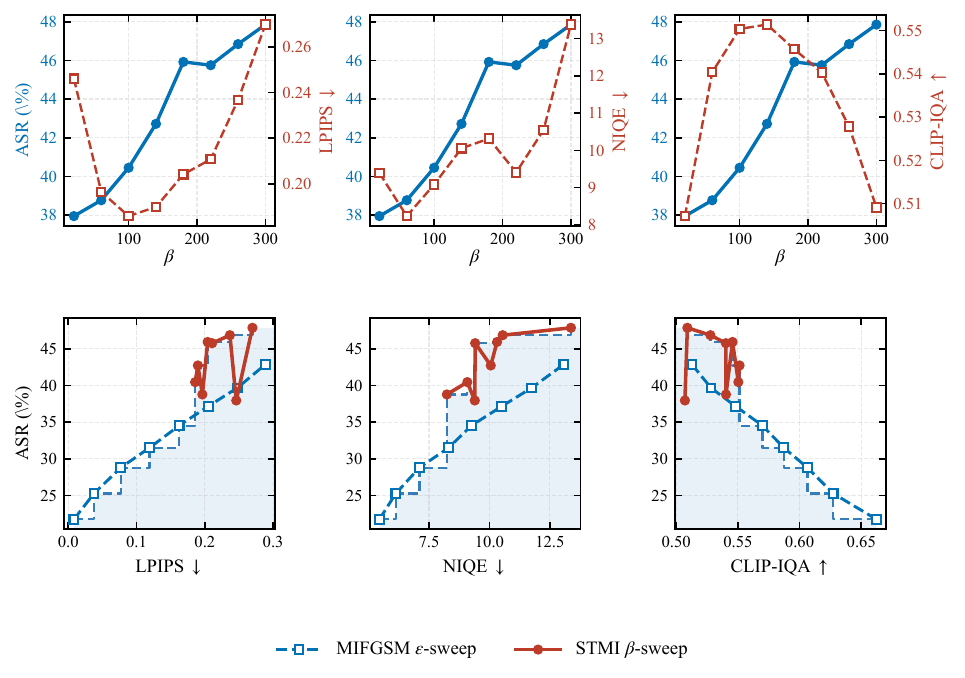}%
\caption{$\beta$ sensitivity of STMI. The default $\beta=160$ comes from the coupling heuristic $\alpha\beta\lesssim 1$, derived in Appendix~\ref{sec:appendix_theory} (Remark~\ref{rem:beta_160}).}
\label{fig:hparam_beta}
\end{figure}

\begin{figure}[h]
\centering
    \includegraphics[width=\columnwidth]{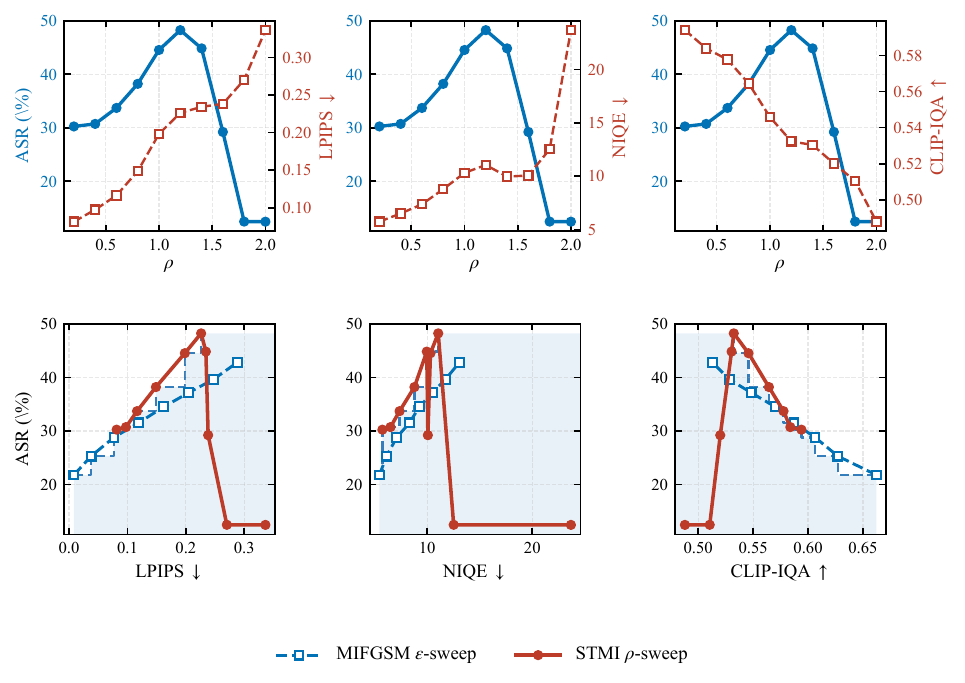}%
\caption{$\rho$ sensitivity of STMI. The sweep mainly tests how far over-relaxation can be pushed before leaving the useful local trade-off segment.}
\label{fig:hparam_rho}
\end{figure}

\textbf{Choice of soft regularizer}
We finally compare three choices of the soft regularizer within the same Condat--V\~{u} template: $\|\deltaadv\|_{\infty}$ (ours), $\|\deltaadv\|_{2}$, and $\|\deltaadv\|_{1}$. By Fenchel duality, these correspond to $L_1$, $L_2$, and $L_\infty$ dual-ball projections, respectively. We compare these three objective-level choices while keeping the base attack, $\epsilon$, step schedule, iteration budget, and surrogate fixed.

Table~\ref{tab:lp_ablation} shows that the $\|\deltaadv\|_{\infty}$ soft regularizer used by ST gives the best overall balance among the three tested penalties. It attains the highest average ASR and the best value on nearly all metrics, with NIQE as the main exception. The $\|\deltaadv\|_{1}$ penalty gives the best NIQE, but that advantage comes with weaker results on most other metrics. We therefore keep the $\|\deltaadv\|_{\infty}$ regularizer as the default, since it most consistently improves the overall ASR--imperceptibility trade-off across complementary metrics.
\begin{table}[t]
\caption{Soft-regularizer ablation on STMI. The three columns correspond to three objective choices for $h(\deltaadv)$, implemented by their dual projections inside the same Condat--V\~{u} template; best per row in bold.}
\label{tab:lp_ablation}
\centering
\small
\setlength{\tabcolsep}{4pt}
\begin{tabular}{@{}l r r r@{}}
\toprule
Metric & $\lambda\|\deltaadv\|_{\infty}$ (ours) & $\lambda\|\deltaadv\|_{2}$ & $\lambda\|\deltaadv\|_{1}$ \\
\midrule
ASR $\uparrow$ (\%) & \textbf{44.19} & 44.58 & 42.44 \\
PSNR $\uparrow$ & \textbf{29.48} & 28.35 & 28.70 \\
MS-SSIM $\uparrow$ & \textbf{0.953} & 0.936 & 0.941 \\
LPIPS $\downarrow$ & \textbf{0.198} & 0.258 & 0.241 \\
DISTS $\downarrow$ & \textbf{0.189} & 0.206 & 0.198 \\
TOPIQ-FR $\uparrow$ & \textbf{0.701} & 0.615 & 0.625 \\
LF $\downarrow$ & \textbf{0.0168} & 0.0189 & 0.0171 \\
CLIP-IQA $\uparrow$ & \textbf{0.546} & 0.520 & 0.529 \\
NIQE $\downarrow$ & 10.36 & 9.52 & \textbf{7.68} \\
FID $\downarrow$ & \textbf{57.65} & 68.62 & 64.50 \\
\bottomrule
\end{tabular}

\end{table}

\begin{figure*}
	\centering
		\includegraphics[width=0.95\textwidth]{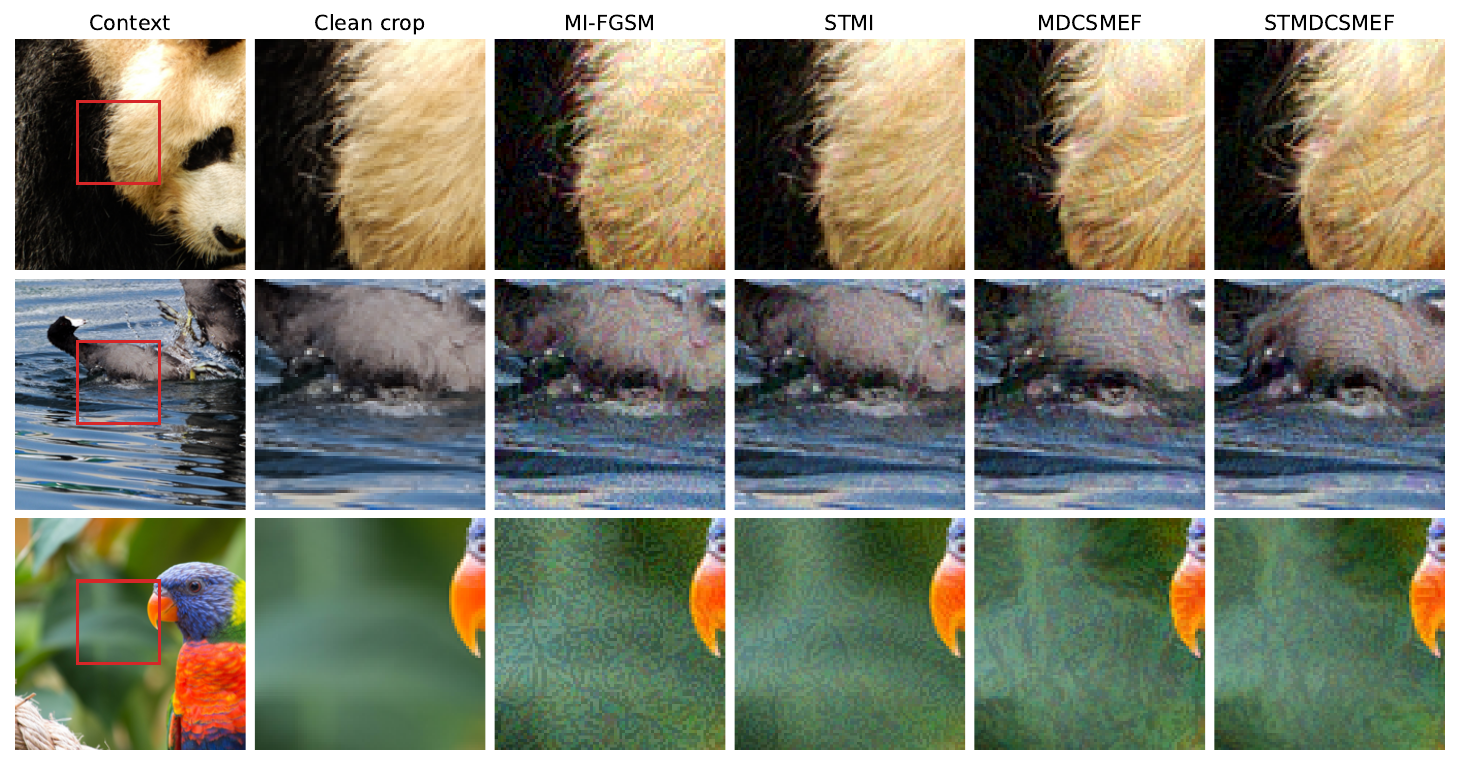}%
	\caption{Qualitative comparison at fixed $\epsilon{=}16/255$ (the first three rows of Figure~\ref{fig:appendix_same_eps_page_visual}). Each row shows the clean-image context with the selected crop region marked in red, the clean crop itself, and the corresponding crop regions for MI-FGSM, STMI, MDCSMEF, and STMDCSMEF. At the same perturbation budget, the ST variants suppress local color speckle and edge ringing while preserving cleaner object boundaries.}
	\label{fig:visual_crops}
\end{figure*}

\subsection{Visualization}
\label{subsec:visualization}
We visualize MI-FGSM and MDCSMEF together with their ST-wrapped versions. Figure~\ref{fig:visual_crops} shows three representative samples, together with the original-image context and the local zoomed regions used for inspection. A full-page same-$\epsilon$ view that pairs each adversarial image with its $5\times\delta$ perturbation, on these three samples plus one additional sample, is provided in Appendix~\ref{subsec:appendix_visual} (Figure~\ref{fig:appendix_same_eps_page_visual}); the matched-ASR counterpart on the same four samples, in which each method uses the $\epsilon$ that lands its ASR near $50\%$, is given by Figure~\ref{fig:appendix_matched_asr_page_visual} of the same appendix and complements the same-$\epsilon$ view from the Pareto-frontier angle.

\section{Discussion and Limitations}
\label{sec:discussion}

The empirical evidence in Section~\ref{sec:experiments} has established that ST shifts every tested base attack outward on the ASR--imperceptibility Pareto frontier. We now step back from the numbers and place ST in two contexts: a methodological one that explains \emph{why} the wrapper has this effect, and a candid statement of what its empirical claims do and do not cover.

\subsection{Discussion}
\label{subsec:discussion_main}

\noindent\textbf{(a) Why a soft $L_\infty$ regularizer.} The Pareto-frontier reading in \cref{sec:experiments} dictates the form of the wrapper. Stronger transfer attacks already lie on a more favorable ASR--imperceptibility curve than weaker ones, so what we want is not a new attack design but a single scalar $\lambda$ that smoothly traces the curve outward without disturbing the base attack. This points to a loss-plus-regularization shape: leave the direction-generation machinery intact and add an imperceptibility term whose strength is tuned by one parameter. The specific choice of $\|\deltaadv\|_\infty$ as the regularizer is empirical: the ablation in \cref{tab:lp_ablation} shows that, within the same Condat--V\~{u} template, $\|\deltaadv\|_\infty$ dominates $\|\deltaadv\|_1$ and $\|\deltaadv\|_2$ on most stealth metrics at comparable ASR.

\noindent\textbf{(b) The $L_\infty$ regularizer does not reduce $\|\deltaadv\|_\infty$, yet perceptual metrics still improve.} Across most results in \cref{sec:experiments}, both ST and the corresponding base attacks reach $\|\deltaadv\|_\infty=\epsilon$, so the regularizer does not visibly reduce the peak perturbation magnitude. This is expected as once the perturbation reaches the boundary of $\Q$, the empirical $\|\deltaadv\|_\infty$ is pinned at $\epsilon$ by construction, and box-constrained gradient methods typically saturate at least one coordinate. However, several stealth metrics are still substantially improved. A plausible explanation is that the regularizer affects the distribution of the perturbation rather than its maximum value: it reduces the number of saturated coordinates and pulls part of the mass back into the interior, thereby mitigating the structured artifacts that these metrics capture.

\subsection{Limitations}

\textbf{Metric proxies do not equal human perception.} Imperceptibility in this paper is reported through objective image-quality indices rather than through direct human judgment. We draw on a broad and complementary set of metrics—full-reference, no-reference, and distributional. However, none of these proxies coincides perfectly with human perception. Each indicator emphasizes a different aspect of distortion. Reporting trade-offs across complementary metrics mitigates but does not eliminate this gap; human-subject studies remain the natural next step.

\textbf{Physical-world scope.} Our evaluation is limited to digital adversarial perturbations. Whether the same Pareto shift holds when perturbations must survive printing, photography, and scene-level variation remains an open question. Physically realizable adversarial attacks differ substantially from the digital pixel-bounded threat model studied here and are therefore left to future work.

\section{Conclusion}
\label{sec:conclusion}

Transferability and imperceptibility have long been pursued and reported separately, with the two goals widely seen as conflicting. This apparent conflict is largely a snapshot effect of the single-$\epsilon$ evaluation protocol: read along the ASR--imperceptibility Pareto frontier, a stronger transfer attack reaches a given ASR with a smaller perturbation budget, and is therefore already more stealthy than a weaker one at matched ASR. We accordingly advocate evaluating transfer attacks along the ASR--imperceptibility Pareto frontier rather than at a single $\epsilon$.

Building on this finding, we introduce ST, a plug-in primal-dual wrapper that recasts adversarial-example generation as a loss-plus-regularization problem. ST keeps the attack loss and the saturation regularizer as separate terms and handles them through independent primal and dual updates, so that convergence and imperceptibility control are addressed by their own dedicated operators rather than fused into a single gradient step. The wrapper modifies only the update step of a compatible base attack and requires no auxiliary trained module or handcrafted perceptual prior.

Empirically, ST inherits the transferability of all seven representative base attacks that we tested, shifts each of them outward on the ASR--imperceptibility frontier, and at matched high-ASR levels Pareto-dominates dedicated stealth-oriented transfer attacks that rely on auxiliary perceptual signals or trained modules. This pattern reinforces the design principle behind ST: it is more effective to inherit a strong transferable attack and reshape only its update step than to build a bespoke stealth-oriented attack from scratch on top of a weaker base, since the wrapper then composes with the strongest direction-generation machinery already available rather than competing against it.

This Pareto-frontier perspective, together with ST, calls on future work on transferable adversarial attacks to evaluate transferability and imperceptibility jointly, report operating curves rather than single-$\epsilon$ snapshots, and assess imperceptibility through multiple complementary metrics.

\section*{Acknowledgments}
This work was supported by the NSFC (Grant No. 62576351). 
AI-assisted tools were used during the preparation of this manuscript for language editing, code development. All content was reviewed and verified by the authors.

\bibliographystyle{IEEEtran}
\bibliography{references}

@misc{szegedy2014intriguing,
  author        = {Szegedy, Christian and Zaremba, Wojciech and Sutskever, Ilya and Bruna, Joan and Erhan, Dumitru and Goodfellow, Ian and Fergus, Rob},
  title         = {{Intriguing Properties of Neural Networks}},
  year          = {2014},
  eprint        = {1312.6199},
  archivePrefix = {arXiv},
  primaryClass  = {cs.CV},
  howpublished  = {International Conference on Learning Representations (ICLR)}
}

@inproceedings{madry2018pgd,
  author        = {Madry, Aleksander and Makelov, Aleksandar and Schmidt, Ludwig and Tsipras, Dimitris and Vladu, Adrian},
  title         = {{Towards Deep Learning Models Resistant to Adversarial Attacks}},
  booktitle     = {Proceedings of the International Conference on Learning Representations (ICLR)},
  year          = {2018},
  eprint        = {1706.06083},
  archivePrefix = {arXiv},
  primaryClass  = {stat.ML}
}

@inproceedings{dong2018mifgsm,
  title={Boosting adversarial attacks with momentum},
author={Dong, Yinpeng and Liao, Fangzhou and Pang, Tianyu and Su, Hang and Zhu, Jun and Hu, Xiaolin and Li, Jianguo},
booktitle={Proceedings of the IEEE conference on computer vision and pattern recognition},
pages={9185--9193},
year={2018}
}

@inproceedings{lin2020nifgsm,
  author    = {Lin, Jiadong and Song, Chuanbiao and He, Kun and Wang, Liwei and Hopcroft, John E.},
  title     = {{Nesterov Accelerated Gradient and Scale Invariance for Adversarial Attacks}},
  booktitle = {Proceedings of the International Conference on Learning Representations (ICLR)},
  year      = {2020},
  eprint    = {1908.06281},
  archivePrefix = {arXiv},
  primaryClass  = {cs.LG}
}

@inproceedings{wang2021vmi,
  title={Enhancing the transferability of adversarial attacks through variance tuning},
author={Wang, Xiaosen and He, Kun},
booktitle={Proceedings of the IEEE/CVF conference on computer vision and pattern recognition},
pages={1924--1933},
year={2021}
}

@misc{wang2021emi,
  author        = {Wang, Xiaosen and Lin, Jiadong and Hu, Han and Wang, Jingdong and He, Kun},
  title         = {{Boosting Adversarial Transferability through Enhanced Momentum}},
  year          = {2021},
  eprint        = {2103.10609},
  archivePrefix = {arXiv},
  primaryClass  = {cs.CV},
  howpublished  = {British Machine Vision Conference (BMVC)}
}

@inproceedings{zhu2023gra,
  title={Boosting adversarial transferability via gradient relevance attack},
author={Zhu, Hegui and Ren, Yuchen and Sui, Xiaoyan and Yang, Lianping and Jiang, Wuming},
booktitle={Proceedings of the IEEE/CVF international conference on computer vision},
pages={4741--4750},
year={2023}
}

@inproceedings{qin2022rap,
  author    = {Qin, Zeyu and Fan, Yanbo and Liu, Yi and Shen, Li and Zhang, Yong and Wang, Jue and Wu, Baoyuan},
  title     = {{Boosting the Transferability of Adversarial Attacks with Reverse Adversarial Perturbation}},
  booktitle = {Advances in Neural Information Processing Systems (NeurIPS)},
  year      = {2022}
}

@article{qiu2026mef,
  title={Boosting Adversarial Transferability with Low-Cost Optimization via Maximin Expected Flatness},
author={Qiu, Chunlin and Li, Ang and Duan, Yiheng and Zhang, Shenyi and Zhang, Yuanjie and Zhao, Lingchen and Wang, Qian},
journal={IEEE Transactions on Information Forensics and Security},
year={2026},
publisher={IEEE}
}

@inproceedings{tao2026mdcs,
  author        = {Tao, Wei and Dai, Yang and Huang, Jincai and Tao, Qing},
  title         = {{The Power of Decaying Steps: Enhancing Attack Stability and Transferability for Sign-based Optimizers}},
  booktitle = {Proceedings of the IEEE/CVF Conference on Computer Vision and Pattern Recognition (CVPR)},
  year={2026}
  }

@inproceedings{xie2019dim,
  author    = {Xie, Cihang and Zhang, Zhishuai and Zhou, Yuyin and Bai, Song and Wang, Jianyu and Ren, Zhou and Yuille, Alan L.},
  title     = {{Improving Transferability of Adversarial Examples with Input Diversity}},
  booktitle = {Proceedings of the IEEE/CVF Conference on Computer Vision and Pattern Recognition (CVPR)},
  pages     = {2730--2739},
  year      = {2019},
  doi       = {10.1109/CVPR.2019.00284}
}

@inproceedings{dong2019tim,
  author    = {Dong, Yinpeng and Pang, Tianyu and Su, Hang and Zhu, Jun},
  title     = {{Evading Defenses to Transferable Adversarial Examples by Translation-Invariant Attacks}},
  booktitle = {Proceedings of the IEEE/CVF Conference on Computer Vision and Pattern Recognition (CVPR)},
  pages     = {4312--4321},
  year      = {2019},
  doi       = {10.1109/CVPR.2019.00444}
}

@inproceedings{guo2025ops,
  author    = {Guo, Yu and Liu, Weiquan and Xu, Qingshan and Zheng, Shijun and Huang, Shujun and Zang, Yu and Shen, Siqi and Wen, Chenglu and Wang, Cheng},
  title     = {{Boosting Adversarial Transferability through Augmentation in Hypothesis Space}},
  booktitle = {Proceedings of the IEEE/CVF Conference on Computer Vision and Pattern Recognition (CVPR)},
  pages     = {19175--19185},
  year      = {2025},
  doi       = {10.1109/CVPR52734.2025.01786}
}

@inproceedings{chen2018ead,
  author    = {Chen, Pin-Yu and Sharma, Yash and Zhang, Huan and Yi, Jinfeng and Hsieh, Cho-Jui},
  title     = {{EAD: Elastic-Net Attacks to Deep Neural Networks via Adversarial Examples}},
  booktitle = {Proceedings of the AAAI Conference on Artificial Intelligence (AAAI)},
  pages     = {10--17},
  year      = {2018},
  eprint    = {1709.04114},
  archivePrefix = {arXiv},
  primaryClass  = {stat.ML}
}

@inproceedings{modas2019sparsefool,
  author    = {Modas, Apostolos and Moosavi-Dezfooli, Seyed-Mohsen and Frossard, Pascal},
  title     = {{SparseFool: A Few Pixels Make a Big Difference}},
  booktitle = {Proceedings of the IEEE/CVF Conference on Computer Vision and Pattern Recognition (CVPR)},
  pages     = {9087--9096},
  year      = {2019},
  doi       = {10.1109/CVPR.2019.00930}
}

@inproceedings{croce2021l1apgd,
  author    = {Croce, Francesco and Hein, Matthias},
  title     = {{Mind the Box: $\ell_1$-APGD for Sparse Adversarial Attacks on Image Classifiers}},
  booktitle = {Proceedings of the International Conference on Machine Learning (ICML)},
  year      = {2021},
  eprint    = {2103.01208},
  archivePrefix = {arXiv},
  primaryClass  = {cs.LG}
}

@inproceedings{xu2019strattack,
  author    = {Xu, Kaidi and Liu, Sijia and Zhao, Pu and Chen, Pin-Yu and Zhang, Huan and Fan, Quanfu and Erdogmus, Deniz and Wang, Yanzhi and Lin, Xue},
  title     = {{Structured Adversarial Attack: Towards General Implementation and Better Interpretability}},
  booktitle = {Proceedings of the International Conference on Learning Representations (ICLR)},
  year      = {2019},
  eprint    = {1808.01664},
  archivePrefix = {arXiv},
  primaryClass  = {cs.LG}
}

@article{imtiaz2025saif,
  author  = {Imtiaz, Tooba and Kohler, Morgan R. and Miller, Jared F. and Wang, Zifeng and Eskandar, Masih and Sznaier, Mario and Camps, Octavia and Dy, Jennifer},
  title   = {{SAIF: Sparse Adversarial and Imperceptible Attack Framework}},
  journal = {Transactions on Machine Learning Research (TMLR)},
  year    = {2025},
  url     = {https://openreview.net/forum?id=YZL29eJ5j1}
}

@inproceedings{wang2024egstssa,
  author    = {Ming, Di and Ren, Peng and Wang, Yunlong and Feng, Xin},
  title     = {{Transferable Structural Sparse Adversarial Attack via Exact Group Sparsity Training}},
  booktitle = {Proceedings of the IEEE/CVF Conference on Computer Vision and Pattern Recognition (CVPR)},
  pages     = {24696--24705},
  year      = {2024}
}

@misc{zimmer2023gse,
  author        = {Sadiku, Shpresim and Wagner, Moritz and Pokutta, Sebastian},
  title         = {{GSE: Group-wise Sparse and Explainable Adversarial Attacks}},
  year          = {2023},
  eprint        = {2311.17434},
  archivePrefix = {arXiv},
  primaryClass  = {cs.LG}
}

@inproceedings{fan2020sapf,
  author    = {Fan, Yanbo and Wu, Baoyuan and Li, Tuanhui and Zhang, Yong and Li, Mingyang and Li, Zhifeng and Yang, Yujiu},
  title     = {{Sparse Adversarial Attack via Perturbation Factorization}},
  booktitle = {Proceedings of the European Conference on Computer Vision (ECCV)},
  pages     = {35--50},
  year      = {2020},
  doi       = {10.1007/978-3-030-58542-6\_3}
}

@inproceedings{zhao2020percal,
  author    = {Zhao, Zhengyu and Liu, Zhuoran and Larson, Martha},
  title     = {{Towards Large Yet Imperceptible Adversarial Image Perturbations with Perceptual Color Distance}},
  booktitle = {Proceedings of the IEEE/CVF Conference on Computer Vision and Pattern Recognition (CVPR)},
  pages     = {1039--1048},
  year      = {2020},
  doi       = {10.1109/CVPR42600.2020.00112}
}

@inproceedings{sharif2018suitability,
	author    = {Sharif, Mahmood and Bauer, Lujo and Reiter, Michael K.},
	title     = {On the Suitability of Lp-Norms for Creating and Preventing
	Adversarial Examples},
	booktitle = {Proceedings of the IEEE Conference on Computer Vision and
	Pattern Recognition (CVPR) Workshops},
	pages     = {1605--1613},
	year      = {2018}
}

@inproceedings{luo2022ssah,
  author    = {Luo, Cheng and Lin, Qinliang and Xie, Weicheng and Wu, Bizhu and Xie, Jinheng and Shen, Linlin},
  title     = {{Frequency-Driven Imperceptible Adversarial Attack on Semantic Similarity}},
  booktitle = {Proceedings of the IEEE/CVF Conference on Computer Vision and Pattern Recognition (CVPR)},
  pages     = {15315--15324},
  year      = {2022},
  doi       = {10.1109/CVPR52688.2022.01488}
}

@article{du2023ps,
  author  = {Zhang, Yaoyuan and Tan, Yu-an and Sun, Haipeng and Zhao, Yuhang and Zhang, Quanxing and Li, Yuanzhang},
  title   = {{Improving the Invisibility of Adversarial Examples with Perceptually Adaptive Perturbation}},
  journal = {Information Sciences},
  volume  = {635},
  pages   = {126--137},
  year    = {2023},
  doi     = {10.1016/j.ins.2023.03.139}
}

@inproceedings{hu2024came,
  author    = {Xia, Hui and Zhang, Rui and Kang, Zi and Jiang, Shuliang and Xu, Shuo},
  title     = {{Enhance Stealthiness and Transferability of Adversarial Attacks with Class Activation Mapping Ensemble Attack}},
  booktitle = {Proceedings of the Network and Distributed System Security Symposium (NDSS)},
  year      = {2024}
}

@misc{guesmi2025tesser,
  author        = {Guesmi, Amira and Ouni, Bassem and Shafique, Muhammad},
  title         = {{TESSER: Transfer-Enhancing Adversarial Attacks from Vision Transformers via Spectral and Semantic Regularization}},
  year          = {2025},
  eprint        = {2505.19613},
  archivePrefix = {arXiv},
  primaryClass  = {cs.CV}
}

@article{chen2024diffattack,
  author  = {Chen, Jianqi and Chen, Hao and Chen, Keyan and Zhang, Yilan and Zou, Zhengxia and Shi, Zhenwei},
  title   = {{Diffusion Models for Imperceptible and Transferable Adversarial Attack}},
  journal = {IEEE Transactions on Pattern Analysis and Machine Intelligence (TPAMI)},
  volume  = {47},
  number  = {2},
  pages   = {961--977},
  year    = {2025},
  doi     = {10.1109/TPAMI.2024.3480519}
}

@inproceedings{li2024advad,
  author    = {Li, Jin and He, Ziqiang and Luo, Anwei and Hu, Jian-Fang and Wang, Z. Jane and Kang, Xiangui},
  title     = {{AdvAD: Exploring Non-Parametric Diffusion for Imperceptible Adversarial Attacks}},
  booktitle = {Advances in Neural Information Processing Systems (NeurIPS)},
  year      = {2024}
}

@inproceedings{wang2003msssim,
  author    = {Wang, Zhou and Simoncelli, Eero P. and Bovik, Alan C.},
  title     = {{Multiscale Structural Similarity for Image Quality Assessment}},
  booktitle = {Proceedings of the 37th Asilomar Conference on Signals, Systems and Computers (ACSSC)},
  volume    = {2},
  pages     = {1398--1402},
  year      = {2003},
  doi       = {10.1109/ACSSC.2003.1292216}
}

@inproceedings{zhang2018lpips,
  author    = {Zhang, Richard and Isola, Phillip and Efros, Alexei A. and Shechtman, Eli and Wang, Oliver},
  title     = {{The Unreasonable Effectiveness of Deep Features as a Perceptual Metric}},
  booktitle = {Proceedings of the IEEE Conference on Computer Vision and Pattern Recognition (CVPR)},
  pages     = {586--595},
  year      = {2018},
  doi       = {10.1109/CVPR.2018.00068}
}

@article{ding2020dists,
  author  = {Ding, Keyan and Ma, Kede and Wang, Shiqi and Simoncelli, Eero P.},
  title   = {{Image Quality Assessment: Unifying Structure and Texture Similarity}},
  journal = {IEEE Transactions on Pattern Analysis and Machine Intelligence (TPAMI)},
  volume  = {44},
  number  = {5},
  pages   = {2567--2581},
  year    = {2022},
  doi     = {10.1109/TPAMI.2020.3045810}
}

@article{chen2024topiq,
  author  = {Chen, Chaofeng and Mo, Jiadi and Hou, Jingwen and Wu, Haoning and Liao, Liang and Sun, Wenxiu and Yan, Qiong and Lin, Weisi},
  title   = {{TOPIQ: A Top-Down Approach from Semantics to Distortions for Image Quality Assessment}},
  journal = {IEEE Transactions on Image Processing},
  volume  = {33},
  pages   = {2404--2418},
  year    = {2024},
  doi     = {10.1109/TIP.2024.3378466}
}

@inproceedings{wang2023clipiqa,
  author    = {Wang, Jianyi and Chan, Kelvin C. K. and Loy, Chen Change},
  title     = {{Exploring CLIP for Assessing the Look and Feel of Images}},
  booktitle = {Proceedings of the AAAI Conference on Artificial Intelligence (AAAI)},
  pages     = {2555--2563},
  year      = {2023},
  doi       = {10.1609/aaai.v37i2.25353}
}

@article{mittal2013niqe,
  author  = {Mittal, Anish and Soundararajan, Rajiv and Bovik, Alan C.},
  title   = {{Making a ``Completely Blind'' Image Quality Analyzer}},
  journal = {IEEE Signal Processing Letters},
  volume  = {20},
  number  = {3},
  pages   = {209--212},
  year    = {2013},
  doi     = {10.1109/LSP.2012.2227726}
}

@inproceedings{heusel2017fid,
  author    = {Heusel, Martin and Ramsauer, Hubert and Unterthiner, Thomas and Nessler, Bernhard and Hochreiter, Sepp},
  title     = {{GANs Trained by a Two Time-Scale Update Rule Converge to a Local Nash Equilibrium}},
  booktitle = {Advances in Neural Information Processing Systems (NeurIPS)},
  year      = {2017}
}

@misc{ghazanfari2023rlpips,
  author        = {Ghazanfari, Sara and Garg, Siddharth and Krishnamurthy, Prashanth and Khorrami, Farshad and Araujo, Alexandre},
  title         = {{R-LPIPS: An Adversarially Robust Perceptual Similarity Metric}},
  year          = {2023},
  eprint        = {2307.15157},
  archivePrefix = {arXiv},
  primaryClass  = {cs.CV}
}

@article{ghildyal2023attackingiqa,
  author  = {Ghildyal, Abhijay and Liu, Feng},
  title   = {{Attacking Perceptual Similarity Metrics}},
  journal = {Transactions on Machine Learning Research (TMLR)},
  year    = {2023},
  url     = {https://openreview.net/forum?id=r9vGSpbbRO}
}

@article{sharma2005ciede2000,
  author  = {Sharma, Gaurav and Wu, Wencheng and Dalal, Edul N.},
  title   = {{The CIEDE2000 Color-Difference Formula: Implementation Notes, Supplementary Test Data, and Mathematical Observations}},
  journal = {Color Research \& Application},
  volume  = {30},
  number  = {1},
  pages   = {21--30},
  year    = {2005},
  doi     = {10.1002/col.20070}
}

@article{condat2013primaldual,
  author  = {Condat, Laurent},
  title   = {{A Primal-Dual Splitting Method for Convex Optimization Involving Lipschitzian, Proximable and Linear Composite Terms}},
  journal = {Journal of Optimization Theory and Applications},
  volume  = {158},
  number  = {2},
  pages   = {460--479},
  year    = {2013},
  doi     = {10.1007/s10957-012-0245-9}
}

@article{vu2013splitting,
  author  = {V{\~u}, B{\`a}ng C{\^o}ng},
  title   = {{A Splitting Algorithm for Dual Monotone Inclusions Involving Cocoercive Operators}},
  journal = {Advances in Computational Mathematics},
  volume  = {38},
  number  = {3},
  pages   = {667--681},
  year    = {2013},
  doi     = {10.1007/s10444-011-9254-8}
}

@inproceedings{duchi2008l1projection,
  author    = {Duchi, John and Shalev-Shwartz, Shai and Singer, Yoram and Chandra, Tushar},
  title     = {{Efficient Projections onto the $\ell_1$-Ball for Learning in High Dimensions}},
  booktitle = {Proceedings of the 25th International Conference on Machine Learning (ICML)},
  pages     = {272--279},
  year      = {2008},
  doi       = {10.1145/1390156.1390191}
}

@article{zhao2025revisiting,
	title={Revisiting transferable adversarial images: Systemization, evaluation, and new insights},
	author={Zhao, Zhengyu and Zhang, Hanwei and Li, Renjue and Sicre, Ronan and Amsaleg, Laurent and Backes, Michael and Li, Qi and Wang, Qian and Shen, Chao},
	journal={IEEE Transactions on Pattern Analysis and Machine Intelligence},
	year={2025},
	publisher={IEEE}
}

@article{wu2024aml,
  author  = {Wu, Baoyuan and Zhu, Zihao and Liu, Li and Liu, Qingshan and He, Zhaofeng and Lyu, Siwei},
  title   = {{Attacks in Adversarial Machine Learning: A Systematic Survey from the Life-Cycle Perspective}},
  journal = {arXiv preprint arXiv:2302.09457},
  year    = {2024},
  eprint  = {2302.09457},
  archivePrefix = {arXiv},
  primaryClass  = {cs.LG}
}

@article{liu2024stealthiness,
	title={Stealthiness assessment of adversarial perturbation: From a visual perspective},
	author={Liu, Hangcheng and Zhou, Yuan and Yang, Ying and Zhao, Qingchuan and Zhang, Tianwei and Xiang, Tao},
	journal={IEEE Transactions on Information Forensics and Security},
	volume={20},
	pages={898--913},
	year={2025},
	publisher={IEEE}
}

@book{bauschke2017convex,
  author    = {Bauschke, Heinz H. and Combettes, Patrick L.},
  title     = {Convex Analysis and Monotone Operator Theory in {H}ilbert Spaces},
  publisher = {Springer},
  edition   = {2nd},
  year      = {2017},
  doi       = {10.1007/978-3-319-48311-5}
}

@inproceedings{wu2020skip,
  author    = {Wu, Dongxian and Wang, Yisen and Xia, Shu-Tao and Bailey, James and Ma, Xingjun},
  title     = {Skip Connections Matter: On the Transferability of Adversarial Examples Generated with {ResNets}},
  booktitle = {International Conference on Learning Representations (ICLR)},
  year      = {2020},
  url       = {https://openreview.net/forum?id=BJlRs34Fvr}
}

@inproceedings{lin2011ladmap,
  author    = {Lin, Zhouchen and Liu, Risheng and Su, Zhixun},
  title     = {Linearized Alternating Direction Method with Adaptive Penalty for Low-Rank Representation},
  booktitle = {Advances in Neural Information Processing Systems (NIPS)},
  pages     = {612--620},
  year      = {2011}
}

@book{lin2022admm_book,
  author    = {Lin, Zhouchen and Li, Huan and Fang, Cong},
  title     = {Alternating Direction Method of Multipliers for Machine Learning},
  publisher = {Springer Singapore},
  year      = {2022},
  doi       = {10.1007/978-981-16-9840-8},
  isbn      = {978-981-16-9839-2}
}

\clearpage
\appendices
\newpage
\section{Algorithmic Derivation and Convergence Discussion}
\label{sec:appendix_theory}

This appendix fills in the derivation of ST summarized in Section~\ref{sec:method} (Algorithm~\ref{alg:st}), states a convergence guarantee under an idealized smooth-convex surrogate model, explains the practical $\alpha\beta$ coupling heuristic used throughout the experiments, and provides the explicit update rule (Algorithm~\ref{alg:st_mdcs}) used by the MDCS-based variants.

\subsection{From the Regularized Objective to a Saddle Formulation}
We work in the perturbation space; recall that the feasible set
\begin{equation}
\Q=\left\{\deltaadv:\|\deltaadv\|_{\infty}\le\epsilon,\; \xclean+\deltaadv\in[0,1]^d\right\}
\end{equation}
is the intersection of the $L_\infty$-ball of radius $\epsilon$ with the image-domain box $[0,1]^d$ recentered at $\xclean$. The starting point of the derivation is the regularized formulation~\eqref{eq:st_main_problem} introduced in Section~\ref{sec:method}, where the budget and image-range constraints are folded into the indicator $\iota_{\Q}$ so that the two non-smooth terms—the feasibility indicator $\iota_{\Q}$ and the $L_\infty$ regularizer $\lambda\|\cdot\|_{\infty}$—are made explicit. For compactness, define
\begin{equation}
f(\deltaadv)\triangleq -J(\deltaadv),\qquad
g(\deltaadv)\triangleq \iota_{\Q}(\deltaadv),\qquad
h(\deltaadv)\triangleq \lambda\|\deltaadv\|_{\infty},
\end{equation}
where $J(\deltaadv)$ is the attack objective evaluated at $\xclean+\deltaadv$. With this notation, problem~\eqref{eq:st_main_problem} can be written compactly as
\begin{equation}
\label{eq:app_primal_problem}
\min_{\deltaadv}\; f(\deltaadv)+g(\deltaadv)+h(\deltaadv),
\end{equation}
in which $f$ is smooth and both $g$ and $h$ are non-smooth. The next proposition expresses the $L_\infty$ regularizer $h$ in dual form.

\begin{proposition}[Support-function representation of the $L_\infty$ regularizer]
\label{prop:linf_support}
For any $\lambda\ge 0$, let $\Qone\triangleq\{\bm{y}\in\R^d:\|\bm{y}\|_1\le\lambda\}$. Then
\begin{equation}
\lambda\|\deltaadv\|_{\infty}=\max_{\bm{y}\in\Qone}\langle \deltaadv,\bm{y}\rangle,
\end{equation}
i.e., $\lambda\|\cdot\|_{\infty}$ is the support function of $\Qone$, which is the standard duality between the $L_1$ and $L_\infty$ norms.
\end{proposition}

\begin{proof}
For any $\bm{y}\in\Qone$, Hölder's inequality gives
\begin{equation}
\langle \deltaadv,\bm{y}\rangle \le \|\bm{y}\|_1\|\deltaadv\|_{\infty} \le \lambda\|\deltaadv\|_{\infty},
\end{equation}
so the maximum is at most $\lambda\|\deltaadv\|_{\infty}$. To attain the bound, pick any $i^*\in\arg\max_{i}|\deltaadv_i|$ and set
\begin{equation}
\bm{y}^*=\lambda\,\mathrm{sign}(\deltaadv_{i^*})\,\bm{e}_{i^*},
\end{equation}
where $\bm{e}_{i^*}$ is the $i^*$-th canonical basis vector and we adopt the convention $\mathrm{sign}(0)\triangleq 0$. Then $\|\bm{y}^*\|_1=\lambda|\mathrm{sign}(\deltaadv_{i^*})|\le\lambda$, hence $\bm{y}^*\in\Qone$, and
\begin{equation}
\langle \deltaadv,\bm{y}^*\rangle=\lambda|\deltaadv_{i^*}|=\lambda\|\deltaadv\|_{\infty}.
\end{equation}
The degenerate case $\deltaadv=\bm{0}$ is recovered automatically: the convention $\mathrm{sign}(0)=0$ yields $\bm{y}^*=\bm{0}\in\Qone$, and both sides of the identity equal zero.
\end{proof}

By Proposition~\ref{prop:linf_support}, the primal problem~\eqref{eq:app_primal_problem} is equivalent to the saddle problem
\begin{equation}
\label{eq:app_saddle_direct}
\min_{\deltaadv}\max_{\bm{y}\in\Qone}\;
f(\deltaadv)+g(\deltaadv)+\langle \deltaadv,\bm{y}\rangle.
\end{equation}

\begin{proposition}[Convex conjugate of the $L_\infty$ regularizer]
\label{prop:conjugate_linf}
For any $\lambda\ge 0$, the convex conjugate of $h(\deltaadv)=\lambda\|\deltaadv\|_{\infty}$ is
\begin{equation}
h^*(\bm{y})=\iota_{\Qone}(\bm{y}),
\qquad
\Qone=\{\bm{y}\in\R^d:\|\bm{y}\|_1\le\lambda\}.
\end{equation}
\end{proposition}

\begin{proof}
By definition,
\begin{equation}
h^*(\bm{y})=\sup_{\bm{z}}\left\{\langle \bm{y},\bm{z}\rangle-\lambda\|\bm{z}\|_{\infty}\right\}.
\end{equation}
If $\|\bm{y}\|_1\le\lambda$, then Hölder's inequality gives
\begin{equation}
\langle \bm{y},\bm{z}\rangle \le \|\bm{y}\|_1\|\bm{z}\|_{\infty}\le \lambda\|\bm{z}\|_{\infty},
\end{equation}
so the supremum is upper bounded by $0$ and is attained at $\bm{z}=0$. Hence $h^*(\bm{y})=0$.

If $\|\bm{y}\|_1>\lambda$, then in particular $\bm{y}\neq\bm{0}$, so $\|\mathrm{sign}(\bm{y})\|_{\infty}=1$ where $\mathrm{sign}(0)=0$. Choose $\bm{z}=t\,\mathrm{sign}(\bm{y})$ with $t>0$, hence $\|\bm{z}\|_{\infty}=t$ and $\langle\bm{y},\bm{z}\rangle=t\|\bm{y}\|_1$. Therefore
\begin{equation}
\langle \bm{y},\bm{z}\rangle-\lambda\|\bm{z}\|_{\infty}
= t(\|\bm{y}\|_1-\lambda)\to+\infty
\end{equation}
as $t\to\infty$. Therefore $h^*(\bm{y})=+\infty$. Combining the two cases proves the claim.
\end{proof}

Using Proposition~\ref{prop:conjugate_linf}, since $h^*(\bm{y})=\iota_{\Qone}(\bm{y})$ takes value $+\infty$ outside $\Qone$, the dual constraint can be absorbed into $-h^*$ and the saddle problem~\eqref{eq:app_saddle_direct} is equivalent to its unconstrained-dual form
\begin{equation}
\label{eq:app_saddle_conjugate}
\min_{\deltaadv}\max_{\bm{y}\in\R^d}\;
f(\deltaadv)+g(\deltaadv)+\langle \deltaadv,\bm{y}\rangle-h^*(\bm{y}).
\end{equation}

\subsection{From the Saddle Problem to the Primal-Dual Iteration}
\label{subsec:saddle_to_iteration}

This subsection derives the exact-gradient Condat--V\~{u} iteration from the saddle problem~\eqref{eq:app_saddle_conjugate} in four steps: (i) the composite structure of the saddle objective (Proposition~\ref{prop:cv_structure}); (ii) its first-order optimality conditions (Proposition~\ref{prop:optimality_conditions}); (iii) the equivalent fixed-point characterization (Proposition~\ref{prop:fixed_point_form}); and (iv) the resulting primal-dual iteration with the proximal maps specialized to indicator functions (Propositions~\ref{prop:prox_maps}--\ref{prop:exact_st_iteration}).

\begin{proposition}[Composite structure]
\label{prop:cv_structure}
The saddle problem \eqref{eq:app_saddle_conjugate} has the standard primal-dual composite form
\begin{equation}
\min_{\deltaadv}\max_{\vdual}\; f(\deltaadv)+g(\deltaadv)+\langle \deltaadv,\vdual\rangle-h^*(\vdual),
\end{equation}
with one smooth term $f(\deltaadv)=-J(\deltaadv)$ and two non-smooth terms
\begin{equation}
g(\deltaadv)=\iota_{\Q}(\deltaadv),\qquad h^*(\vdual)=\iota_{\Qone}(\vdual).
\end{equation}
Consequently, primal-dual splitting treats $g$ and $h^*$ by independent proximal maps on the primal and dual variables.
\end{proposition}

\begin{proof}
The claim follows directly from the definitions above and Proposition~\ref{prop:conjugate_linf}, which identifies the convex conjugate of $h(\deltaadv)=\lambda\|\deltaadv\|_{\infty}$ as $h^*=\iota_{\Qone}$.
\end{proof}

\begin{proposition}[First-order optimality conditions]
\label{prop:optimality_conditions}
A pair $(\deltaadv^*,\vdual^*)$ is a first-order stationary point of the saddle problem~\eqref{eq:app_saddle_conjugate} if and only if
\begin{align}
\bm{0} &\in \nabla f(\deltaadv^*)+N_{\Q}(\deltaadv^*)+\vdual^*, \label{eq:app_opt_primal}\\
\bm{0} &\in -\deltaadv^*+N_{\Qone}(\vdual^*), \label{eq:app_opt_dual}
\end{align}
where $N_{\Q}(\deltaadv^*)=\partial\iota_{\Q}(\deltaadv^*)$ and $N_{\Qone}(\vdual^*)=\partial\iota_{\Qone}(\vdual^*)$ denote the normal cones of $\Q$ and $\Qone$ at $\deltaadv^*$ and $\vdual^*$, respectively.
\end{proposition}

\begin{proof}
The Lagrangian of \eqref{eq:app_saddle_conjugate} is
\begin{equation}
\mathcal{L}(\deltaadv,\vdual)=f(\deltaadv)+g(\deltaadv)+\langle\deltaadv,\vdual\rangle-h^*(\vdual),
\end{equation}
which is convex in $\deltaadv$ and concave in $\vdual$. A saddle point $(\deltaadv^*,\vdual^*)$ is therefore characterized by the partial subdifferential conditions
\begin{equation}
\bm{0}\in\partial_{\deltaadv}\mathcal{L}(\deltaadv^*,\vdual^*),\qquad \bm{0}\in\partial_{\vdual}(-\mathcal{L})(\deltaadv^*,\vdual^*).
\end{equation}
The primal condition gives
\begin{equation}
\bm{0}\in\nabla f(\deltaadv^*)+\partial g(\deltaadv^*)+\vdual^*,
\end{equation}
and since $g=\iota_{\Q}$, we have $\partial g(\deltaadv^*)=N_{\Q}(\deltaadv^*)$. The dual condition gives
\begin{equation}
\bm{0}\in-\deltaadv^*+\partial h^*(\vdual^*),
\end{equation}
and since $h^*=\iota_{\Qone}$, we have $\partial h^*(\vdual^*)=N_{\Qone}(\vdual^*)$.
\end{proof}

The primal condition~\eqref{eq:app_opt_primal} states that, at a stationary point, the gradient $\nabla f(\deltaadv^*)=-\nabla J(\deltaadv^*)$, the normal cone of $\Q$, and the dual variable cancel each other. The dual condition~\eqref{eq:app_opt_dual} is equivalently the statement that $\vdual^*$ belongs to the subdifferential of the $L_\infty$ regularizer at $\deltaadv^*$.

\begin{proposition}[Fixed-point characterization]
\label{prop:fixed_point_form}
A pair $(\deltaadv^*,\vdual^*)$ satisfies \eqref{eq:app_opt_primal}--\eqref{eq:app_opt_dual} if and only if, for any step sizes $\alpha,\beta>0$,
\begin{align}
\deltaadv^*&=\proj_{\Q}\!\left(\deltaadv^*-\alpha\nabla f(\deltaadv^*)-\alpha\vdual^*\right), \label{eq:app_fp_primal}\\
\vdual^*&=\proj_{\Qone}\!\left(\vdual^*+\beta\deltaadv^*\right). \label{eq:app_fp_dual}
\end{align}
\end{proposition}

\begin{proof}
For a closed convex set $C$ and any $\alpha>0$, the standard subdifferential characterization of the Euclidean projection~\cite[Prop.~6.46]{bauschke2017convex} states that $\bm{w}\in N_{C}(\bm{x})$ if and only if $\bm{x}=\proj_{C}(\bm{x}+\alpha\bm{w})$. Applying this with $C=\Q$ and $\bm{w}=-\nabla f(\deltaadv^*)-\vdual^*$ gives the projection identity~\eqref{eq:app_fp_primal}; applying it with $C=\Qone$ and $\bm{w}=\deltaadv^*$ gives~\eqref{eq:app_fp_dual}.
\end{proof}

The fixed-point system above suggests a natural alternating scheme, but a plain alternation would feed a stale primal iterate into the dual update. Condat~\cite{condat2013primaldual} and V\~{u}~\cite{vu2013splitting} remove this lag via the extrapolated primal point $2\bar{\delta}_{t+1}-\deltaadv_t$, yielding the standard primal-dual iteration
\begin{align}
\bar{\delta}_{t+1} &= \mathrm{prox}_{\alpha g}\!\left(\deltaadv_t-\alpha\nabla f(\deltaadv_t)-\alpha\vdual_t\right), \label{eq:app_cv_primal_general}\\
\bar{\vdual}_{t+1} &= \mathrm{prox}_{\beta h^*}\!\left(\vdual_t+\beta(2\bar{\delta}_{t+1}-\deltaadv_t)\right). \label{eq:app_cv_dual_general}
\end{align}
When $f$ is convex with $L$-Lipschitz-continuous gradient, this Condat--V\~{u} scheme converges under the standard step-size condition stated later in Theorem~\ref{thm:cv_convergence}. We next specialize it to the ST objective.

\begin{proposition}[Proximal maps for indicator functions]
\label{prop:prox_maps}
For any $\alpha,\beta>0$,
\begin{equation}
\mathrm{prox}_{\alpha g}=\proj_{\Q},
\qquad
\mathrm{prox}_{\beta h^*}=\proj_{\Qone}.
\end{equation}
\end{proposition}

\begin{proof}
Since $g=\iota_{\Q}$ and $h^*=\iota_{\Qone}$ are indicator functions of closed convex sets, their proximal operators are exactly the Euclidean projections onto those sets.
\end{proof}

\begin{proposition}[Exact-gradient ST iteration]
\label{prop:exact_st_iteration}
Substituting $K=I$, $\nabla f=-\nabla J$, and the proximal maps from Proposition~\ref{prop:prox_maps} into~\eqref{eq:app_cv_primal_general}--\eqref{eq:app_cv_dual_general} yields
\begin{align}
\bar{\delta}_{t+1} &= \proj_{\Q}\!\left(\deltaadv_t+\alpha\nabla J(\deltaadv_t)-\alpha\vdual_t\right), \label{eq:app_exact_primal}\\
\bar{\vdual}_{t+1} &= \proj_{\Qone}\!\left(\vdual_t+\beta(2\bar{\delta}_{t+1}-\deltaadv_t)\right). \label{eq:app_exact_dual}
\end{align}
\end{proposition}

\begin{proof}
The claim is obtained by direct substitution into the generic Condat--V\~{u} update.
\end{proof}

The practical ST implementation departs from this exact form in two ways: (i) the exact gradient $\nabla J(\deltaadv_t)$ is replaced by a base attack direction $\bm{u}_t$ (e.g., momentum or transformation-averaged gradients), and (ii) an optional relaxation step is applied,
\begin{equation}
\deltaadv_{t+1}=(1-\rho)\deltaadv_t+\rho\bar{\delta}_{t+1},\qquad
\vdual_{t+1}=(1-\rho)\vdual_t+\rho\bar{\vdual}_{t+1}.
\end{equation}
This yields the inexact primal-dual scheme used by all ST variants (Algorithm~\ref{alg:st}).

\subsection{Regularization Scaling Convention}
\label{subsec:app_lambda_scaling}
The theoretical derivation expresses the dual-ball radius as an absolute quantity~$\lambda$, whereas in practice it is more convenient to specify a \emph{dimension-free} scalar
$\bar{\lambda}$ and let the implementation recover the absolute radius via
\begin{equation}
	\label{eq:app_lambda_scaling}
	\lambda = \bar{\lambda}\,d,
	\qquad d = C \times H \times W,
\end{equation}
where $d$ is the perturbation dimensionality ($d = 3 \times 224\times 224$ for ImageNet). This convention ensures that the per-coordinate budget $\lambda / d = \bar{\lambda}$ remains constant across input resolutions, so that a single $\bar{\lambda}$ value transfers between datasets without rescaling. Throughout the paper, the regularization-strength values reported in figures and tables refer to $\bar{\lambda}$, while the absolute dual-ball radius $\lambda$ entering $\Qone$ is recovered via~\eqref{eq:app_lambda_scaling}.

\subsection{Convergence Under an Idealized Smooth-Convex Surrogate Model}
This subsection states the classical Condat--V\~{u} convergence guarantee for an idealized smooth-convex surrogate model.

\begin{assumption}
\label{ass:smooth_convex}
The function $f(\deltaadv)=-J(\deltaadv)$ is convex on $\R^d$ and continuously differentiable with an $L$-Lipschitz continuous gradient.
\end{assumption}

\begin{assumption}
\label{ass:q_convex}
The set $\Q$ is nonempty, closed, and convex.
\end{assumption}

\begin{theorem}
\label{thm:cv_convergence}
Suppose Assumptions~\ref{ass:smooth_convex} and~\ref{ass:q_convex} hold, and let $\{\deltaadv_t,\vdual_t\}$ be generated by the exact-gradient ST iteration with $\rho=1$. If the step sizes satisfy
\begin{equation}
\label{eq:app_stepsize_condition}
\frac{1}{\alpha}-\beta>\frac{L}{2},
\end{equation}
then the iterate sequence $\{(\deltaadv_t,\vdual_t)\}_{t\ge 0}$ converges in $\R^d\times\R^d$ to a primal-dual solution $(\deltaadv^*,\vdual^*)$ of the saddle problem
\begin{equation}
\label{eq:app_saddle}
\min_{\deltaadv\in\R^d}\max_{\bm{y}\in\R^d}\;
f(\deltaadv)+\iota_{\Q}(\deltaadv)+\langle \deltaadv,\bm{y}\rangle-\iota_{\Qone}(\bm{y}).
\end{equation}
Consequently, $\deltaadv^*$ solves the primal problem \eqref{eq:app_primal_problem}.
\end{theorem}

\begin{proof}
Problem~\eqref{eq:app_primal_problem} has the standard Condat--V\~{u} composite form with $K=I$, $g=\iota_{\Q}$, and $h(\deltaadv)=\lambda\|\deltaadv\|_{\infty}$. By Assumption~\ref{ass:smooth_convex}, $f$ is convex and has an $L$-Lipschitz continuous gradient. By Assumption~\ref{ass:q_convex}, $g=\iota_{\Q}$ is proper, closed, and convex, and the set $\Q$ is compact as the intersection of an $L_\infty$-ball with the unit hypercube. The function $h$ is also proper, closed, and convex, and its conjugate $h^*=\iota_{\Qone}$ has the compact effective domain $\Qone$. Hence the Lagrangian
\begin{equation*}
\mathcal{L}(\deltaadv,\vdual)=f(\deltaadv)+g(\deltaadv)+\langle\deltaadv,\vdual\rangle-h^*(\vdual)
\end{equation*}
is convex-concave and continuous on the compact convex product $\Q\times\Qone$, so Sion's minimax theorem guarantees that the saddle problem~\eqref{eq:app_saddle} admits at least one solution. Since $K=I$, we have $\|K\|=1$, so the general Condat--V\~{u} step-size condition
\begin{equation}
\frac{1}{\alpha}-\beta\|K\|^2>\frac{L}{2}
\end{equation}
reduces exactly to~\eqref{eq:app_stepsize_condition}. Under this condition, Theorem~3.1 of Condat~\cite{condat2013primaldual} (equivalently, Theorem~3.1 of V\~{u}~\cite{vu2013splitting}) establishes that the exact-gradient primal-dual iteration~\eqref{eq:app_exact_primal}--\eqref{eq:app_exact_dual} converges to a saddle point of~\eqref{eq:app_saddle}. By Proposition~\ref{prop:optimality_conditions}, the corresponding primal component solves~\eqref{eq:app_primal_problem}.
\end{proof}

\subsection{Practical Interpretation of the \texorpdfstring{$\alpha\beta$}{alpha-beta} Heuristic}
This subsection relates the classical convex step-size condition to the practical coupling heuristic $\alpha\beta\lesssim 1$, and then isolates the pure primal-dual coupling to explain where that heuristic comes from. The step-size condition \eqref{eq:app_stepsize_condition} immediately implies
\begin{equation}
\label{eq:app_alpha_beta}
\alpha\beta < 1-\frac{\alpha L}{2}.
\end{equation}

\begin{remark}
\label{rem:alpha_beta_necessary}
The inequality $\alpha\beta<1$ is a necessary sanity check, but it is not sufficient for convergence unless the smoothness term is negligible. The actual safety margin is reduced by $\alpha L/2$.
\end{remark}

\begin{remark}
\label{rem:alpha_beta_safe}
A convenient sufficient rule is to choose $\alpha<1/L$ and $\alpha\beta\le 1/2$, with at least one inequality strict. Indeed, $\alpha<1/L$ implies $1/\alpha>L$, and $\alpha\beta\le 1/2$ implies $\beta\le 1/(2\alpha)$, hence
\begin{equation}
\frac{1}{\alpha}-\beta \ge \frac{1}{2\alpha} > \frac{L}{2}.
\end{equation}
\end{remark}

To pin down where the rule of thumb comes from, we isolate the pure primal-dual coupling, dropping the gradient term and the projections.

\begin{proposition}[Linearized primal-dual coupling]
\label{prop:linear_coupling}
Consider the simplified ST dynamics with $\rho=1$, inactive projections, and zero attack direction:
\begin{align}
\deltaadv_{t+1} &= \deltaadv_t - \alpha \vdual_t, \\
\vdual_{t+1} &= \vdual_t + \beta(2\deltaadv_{t+1}-\deltaadv_t).
\end{align}
Let $c=\alpha\beta$. Then the linear update matrix is
\begin{equation}
M=
\begin{bmatrix}
1 & -\alpha \\
\beta & 1-2\alpha\beta
\end{bmatrix},
\end{equation}
and for $0<c<1$ its eigenvalues are
\begin{equation}
\lambda_{\pm}=(1-c)\pm i\sqrt{c(1-c)},
\end{equation}
with modulus
\begin{equation}
|\lambda_{\pm}|=\sqrt{1-c}<1.
\end{equation}
\end{proposition}

\begin{proof}
Substituting the primal update $\deltaadv_{t+1}=\deltaadv_t-\alpha\vdual_t$ into the dual update yields
\begin{equation}
\vdual_{t+1}=\vdual_t+\beta\bigl(2(\deltaadv_t-\alpha\vdual_t)-\deltaadv_t\bigr)=\beta\,\deltaadv_t+(1-2\alpha\beta)\,\vdual_t.
\end{equation}
The recurrence can therefore be written as
\begin{equation}
\begin{bmatrix}
\deltaadv_{t+1} \\
\vdual_{t+1}
\end{bmatrix}
=
M
\begin{bmatrix}
\deltaadv_t \\
\vdual_t
\end{bmatrix},
\qquad
M=
\begin{bmatrix}
1 & -\alpha \\
\beta & 1-2\alpha\beta
\end{bmatrix}.
\end{equation}
Let $c=\alpha\beta$. The characteristic polynomial of $M$ is
\begin{equation}
\chi(\lambda)=\lambda^2-2(1-c)\lambda+(1-c).
\end{equation}
For $0<c<1$, the discriminant is
\begin{equation}
\Delta=4(1-c)^2-4(1-c)=-4c(1-c)<0,
\end{equation}
so the eigenvalues are the complex conjugate pair
\begin{equation}
\lambda_{\pm}=(1-c)\pm i\sqrt{c(1-c)}.
\end{equation}
Their squared modulus is
\begin{equation}
|\lambda_{\pm}|^2=(1-c)^2+c(1-c)=1-c,
\end{equation}
which is strictly smaller than $1$ whenever $0<c<1$.
\end{proof}

\begin{remark}
\label{rem:alpha_beta_rule_of_thumb}
Proposition~\ref{prop:linear_coupling} does not prove convergence of the full ST algorithm. It only analyzes the pure primal-dual coupling after removing the gradient term and the projections. It provides a qualitative explanation for the practical rule of thumb $\alpha\beta\lesssim 1$: for $c=\alpha\beta\in(0,1)$ the eigenvalues are complex conjugates with modulus $\sqrt{1-c}<1$, so the linearized dynamics are underdamped and contract toward zero; at $c=1$ the two eigenvalues merge at $0$; and for $c>1$ the eigenvalues become real, $\lambda_\pm=(1-c)\pm\sqrt{c(c-1)}$, with $\lambda_-$ reaching $-1$ exactly at $c=4/3$ and exceeding $1$ in modulus for $c>4/3$, at which point the linear dynamics are unstable. Thus $c=1$ marks the transition from complex oscillatory contraction to real-valued monotone modes, while $c=4/3$ is the sharp stability boundary of the pure linear coupling. The default $\alpha\beta\approx 1.004$ is therefore inside the overdamped regime but not yet unstable in the linearized model; the non-linear damping of $\proj_{\Q}$ and $\proj_{\Qone}$ provides an additional stability margin that is confirmed empirically.
\end{remark}

\begin{remark}
\label{rem:beta_160}
With the default step size $\alpha=1.6/255$, the reciprocal is
\begin{equation}
\frac{1}{\alpha}=\frac{255}{1.6}\approx 159.375.
\end{equation}
Therefore, choosing $\beta=160$ places STMI near the empirical coupling boundary suggested by $\alpha\beta\approx 1$. For STPGD, the primal step size is larger, $\alpha=2.0/255$, i.e., $1.25\times$ the default step size used by the other attacks. Keeping the same coupling heuristic then gives $\beta=128$ in our experiments. The smaller dual step in STPGD is therefore a direct consequence of its larger primal step rather than a separate hyperparameter choice.
\end{remark}

\subsection{Variants with an Adaptive Per-Coordinate Step}
\label{subsec:appendix_stmdcs}

Most ST variants expose a single ascent direction $\bm{u}_t$ that can be plugged into Algorithm~\ref{alg:st} without altering its state. However, the MDCS-based variants~\cite{tao2026mdcs} differ in one structural way: the base attack carries an additional per-coordinate adaptive step $\bm{d}_t\in\R^d$ inherited from the MDCS step-size rule, and this step is element-wise multiplied into $\bm{u}_t$ before the primal update. Algorithm~\ref{alg:st_mdcs} states the resulting iteration, which adds line~5 (the MDCS clipping update for $\bm{d}_t$, with element-wise $\min$, $|\cdot|$, and $\odot$) and replaces $\bm{u}_t$ by the masked direction $\bm{d}_{t+1}\odot\bm{u}_t$ inside the same $\proj_{\Q}$ primal step. The dual and relaxation steps are unchanged. The entries of $\bm{d}_{t+1}$ are non-increasing in $t$ and tighten on coordinates where $|\bm{u}_t|$ has accumulated, so the wrapper preserves the MDCS coordinate-wise damping while adding only the dual correction.

When $\bm{u}_t$ is the $\ell_1$-normalized momentum~\cite{dong2018mifgsm}, the resulting method is STMDCSMI; when $\bm{u}_t$ is the MEF aggregated gradient~\cite{qiu2026mef} obtained by averaging gradients over a neighborhood of conditionally sampled points around $\deltaadv_t$, it is STMDCSMEF. \textbf{Algorithm~\ref{alg:st_mdcs} reduces to Algorithm~\ref{alg:st} in the limit $\gamma\to+\infty$, in which $\bm{d}_t$ remains the all-ones tensor for all $t$.}

\begin{algorithm}[t]
\caption{ST+MDCS update}
\label{alg:st_mdcs}
\begin{algorithmic}[1]
\Require Clean image $\xclean$, label $\ytrue$, feasible set $\Q$, regularization strength $\lambda$ ($\Qone{=}\{\bm{z}:\|\bm{z}\|_1{\le}\lambda\}$), step sizes $\alpha,\beta$, MDCS clipping $\gamma$, relaxation $\rho$, iteration number $T$
\State Initialize $\deltaadv_0=0$, $\vdual_0=0$, $\bm{d}_0=\bm{1}$
\For{$t=0,\dots,T-1$}
\State $\bm{g}_t=\nabla_{\deltaadv}J(\deltaadv_t)$
\State Compute $\bm{u}_t$ from $\bm{g}_t$
\State $\bm{d}_{t+1}=\min\!\left(\tfrac{\gamma}{|\bm{u}_t|},\;\bm{d}_t\right)$
\State $\bar{\delta}_{t+1}=\proj_{\Q}\!\left(\deltaadv_t+\alpha\,(\bm{d}_{t+1}\odot\bm{u}_t)-\alpha\vdual_t\right)$
\State $\bar{\vdual}_{t+1}=\proj_{\Qone}\!\left(\vdual_t+\beta(2\bar{\delta}_{t+1}-\deltaadv_t)\right)$
\State $\deltaadv_{t+1}=(1-\rho)\deltaadv_t+\rho\bar{\delta}_{t+1}$
\State $\vdual_{t+1}=(1-\rho)\vdual_t+\rho\bar{\vdual}_{t+1}$
\EndFor
\State \Return $\xclean+\deltaadv_T$
\end{algorithmic}
\end{algorithm}

\section{Supplementary Evaluation and Reproducibility Details}
\label{sec:appendix_eval_details}

\subsection{Stealth Evaluation}
\label{subsec:appendix_stealth_eval}
This subsection briefly summarizes why stealth evaluation is inherently multi-metric and why the main text reports several complementary quality indices rather than a single score. Evaluating imperceptibility is fundamentally a multi-metric problem: individual quality indices can disagree on adversarial images and can themselves be attacked by optimization~\cite{ghildyal2023attackingiqa}. Full-reference metrics include LPIPS~\cite{zhang2018lpips} for deep-feature perceptual distance, CIEDE2000~\cite{sharma2005ciede2000} for colorimetric difference, and TOPIQ~\cite{chen2024topiq} for top-down semantic-aware quality. No-reference metrics such as NIQE~\cite{mittal2013niqe} and distributional measures such as FID~\cite{heusel2017fid} provide complementary viewpoints, although FID is sensitive to preprocessing and sample size on small adversarial sets. Adversarial-aware metrics such as R-LPIPS~\cite{ghazanfari2023rlpips}, together with direct attacks on LPIPS~\cite{ghildyal2023attackingiqa}, further show that no single perceptual model should dominate the stealth story. For this reason, the main text reports nine complementary indices spanning full-reference (PSNR, MS-SSIM, LPIPS, DISTS, TOPIQ-FR, LF), no-reference (NIQE, CLIP-IQA), and distributional (FID) metrics rather than relying on a single stealth score.

\subsection{Experimental Protocol and Reproducibility Details}
\label{sec:appendix_protocol}

This subsection records the minimal protocol details needed to interpret the reported metrics and qualitative controls.

\noindent\textbf{Stealth-metric computation.}
All metrics are computed on $224\times 224$ inputs. PSNR and MS-SSIM are computed on $[0,1]$ RGB tensors with \texttt{torchmetrics}\footnote{\url{https://github.com/Lightning-AI/torchmetrics}}. LPIPS uses the AlexNet variant of~\cite{zhang2018lpips} and is evaluated after rescaling images to $[-1,1]$. TOPIQ-FR~\cite{chen2024topiq} and DISTS~\cite{ding2020dists} are computed as full-reference metrics through \texttt{pyiqa}\footnote{\url{https://github.com/chaofengc/IQA-PyTorch}}; CLIP-IQA~\cite{wang2023clipiqa} and NIQE~\cite{mittal2013niqe} are computed as no-reference metrics on the adversarial image alone. LF is computed as the mean absolute difference between the Haar-LL subbands of the clean and adversarial images, following the wavelet-domain low-frequency notion of SSAH~\cite{luo2022ssah}. FID~\cite{heusel2017fid} is computed dataset-wise with \texttt{torchmetrics} using 2048-dimensional Inception features on uint8 images in $[0,255]$.

\begin{table}[h]
\caption{Per-attack $\bar\lambda$ used in the main results (Table~\ref{tab:cross_attack_eps16}). Each value is selected from the sweep $\bar\lambda \in [0.1, 5.0]$ to maximize the imperceptibility metrics on the 1000-image evaluation set.}
\label{tab:per_attack_lambda}
\centering
\small
\begin{tabular}{@{}l c@{}}
\toprule
ST variant & $\bar\lambda$ \\
\midrule
STPGD & 0.25 \\
STMI & 2.4 \\
STMDCSMI & 0.9 \\
STDIM & 2.3 \\
STGRA & 3.0 \\
STMDCSMEF & 0.3 \\
STOPS & 3.9 \\
\bottomrule
\end{tabular}

\end{table}

\subsection{Additional Qualitative Controls}
\label{subsec:appendix_visual}

This subsection provides two complementary full-page qualitative controls that read the same four attacks (MI-FGSM, STMI, MDCSMEF, STMDCSMEF) from two angles of the ASR--imperceptibility Pareto frontier. Figure~\ref{fig:appendix_same_eps_page_visual} fixes the perturbation budget at $\epsilon{=}16/255$ and varies the resulting ASR (the conventional same-$\epsilon$ protocol). Figure~\ref{fig:appendix_matched_asr_page_visual} instead fixes the operating ASR at approximately $50\%$ and lets each method use whatever $\epsilon$ it needs to reach that ASR; this is the matched-ASR slice of the Pareto frontier and is the operationally meaningful comparison once one accepts that transfer attacks differ in transferability. In both figures, each row shows the clean image and four attack columns; within each attack column, the upper half shows the adversarial image and the lower half shows the perturbation visualized as $5\times \delta + 0.5$, with clean and adversarial sub-rows kept aligned.

\begin{figure*}[t]
\centering
    \includegraphics[width=0.77\textwidth]{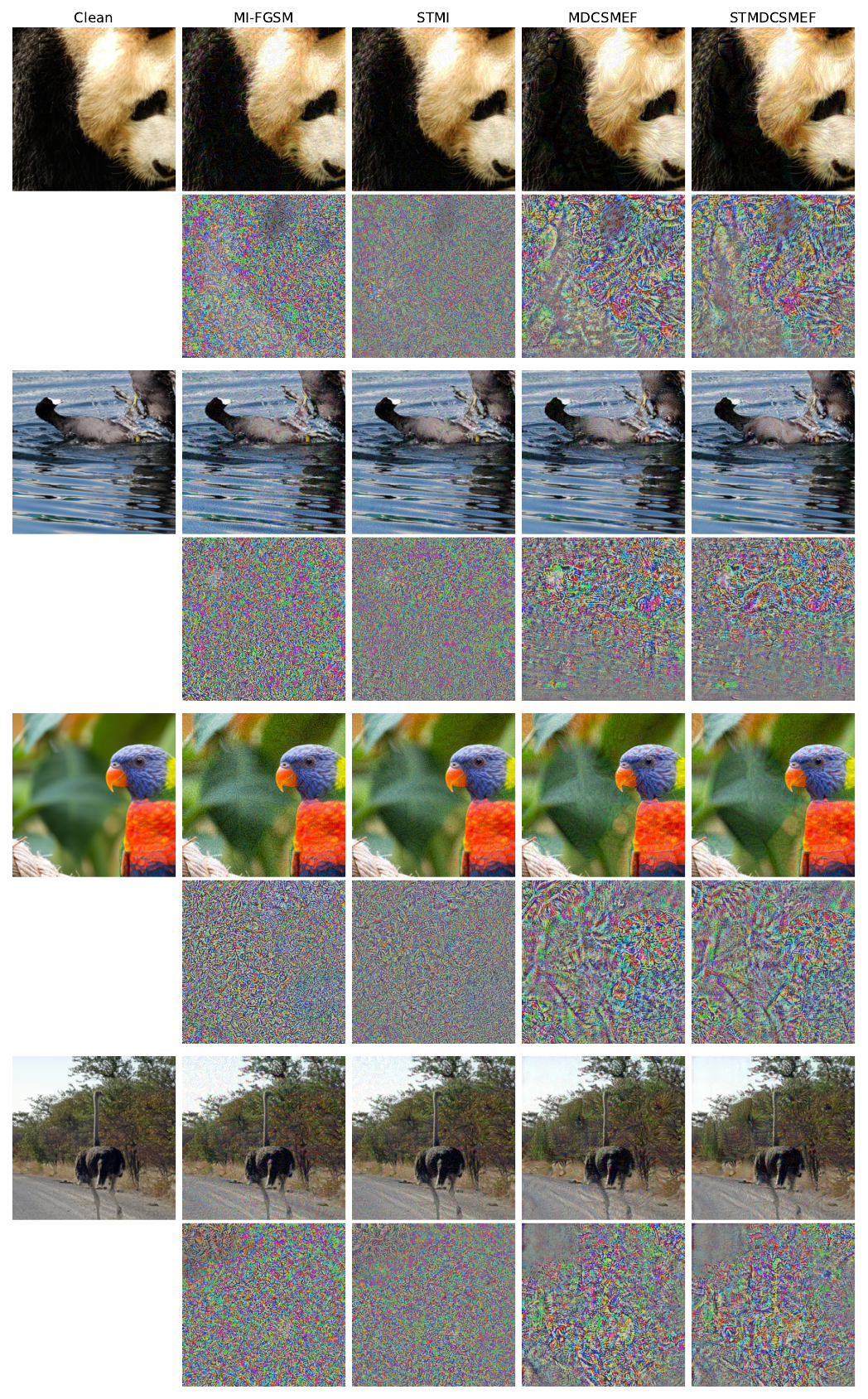}%
\caption{Full-page same-$\epsilon$ qualitative comparison at $\epsilon{=}16/255$ on four representative samples. Each row shows the clean image and four attack columns, where each attack cell contains the adversarial image on top and the perturbation visualized as $5\times \delta + 0.5$ on bottom; the clean image is kept flush with the adversarial sub-row for alignment. The ST variants remain visually cleaner at both the adversarial-image level and the perturbation level.}
\label{fig:appendix_same_eps_page_visual}
\end{figure*}

\begin{figure*}[t]
\centering
    \includegraphics[width=0.77\textwidth]{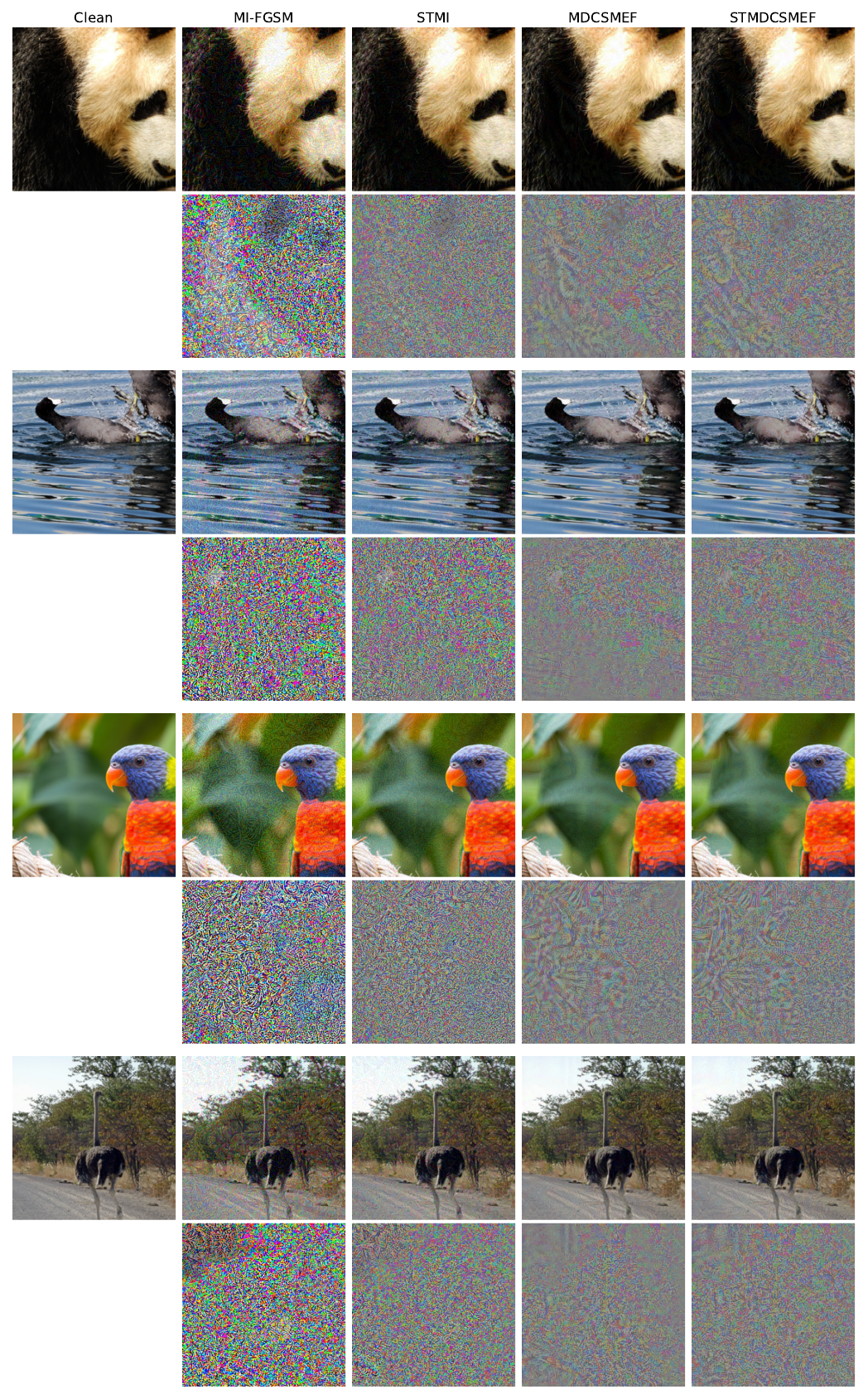}%
\caption{Full-page matched-ASR qualitative comparison at ASR$\approx 50\%$. MI-FGSM at $\epsilon{=}22/255$ (ASR $49.75\%$), STMI at $\epsilon{=}20/255$ (ASR $50.05\%$), MDCSMEF at $\epsilon{=}6/255$ (ASR $50.12\%$), STMDCSMEF at $\epsilon{=}6/255$ (ASR $50.38\%$).}
\label{fig:appendix_matched_asr_page_visual}
\end{figure*}
\clearpage

\end{document}